\documentclass[twoside,11pt]{article}

\usepackage{blindtext}

\usepackage{amsmath,amssymb,amsfonts, enumitem, fancyhdr, color, comment, graphicx, environ,changepage, mathrsfs}
\usepackage{quiver,tikz-cd}
\usepackage[inkscapelatex=false]{svg}
\usepackage{subcaption}
\usepackage{multicol}

\usepackage[preprint]{jmlr2e}

\newtheorem{prop}[theorem]{Proposition}
\newtheorem{coro}[theorem]{Corollary}

\def\gp{\mathcal{GP}}
\def\var{\text{Var}}
\def\cov{\text{Cov}}
\def\erf{\text{erf}}
\def\mean{\mathbb{E}}
\def\x{\mathbf{x}}
\def\y{\mathbf{y}}

\def\xb{\tilde{\mathbf{x}}}

\def\A{\mathbf{A}}
\def\bbeta{\boldsymbol{\beta}}
\def\bgamma{\boldsymbol{\gamma}}
\def\bphi{\boldsymbol{\phi}}
\def\bPhi{\boldsymbol{\Phi}}

\usepackage{lastpage}
\jmlrheading{23}{2024}{1-\pageref{LastPage}}{1/21; Revised 5/22}{9/22}{21-0000}{Kurt Butler, Guanchao Feng, and Petar M. Djuri\'{c}}

\ShortHeadings{Explainable Learning with Gaussian Processes}{Butler, Feng and Djuri\'{c} }
\firstpageno{1}

\begin{document}

\title{Explainable Learning with Gaussian Processes}

\author{\name Kurt Butler \email kurt.butler@stonybrook.edu \\
       \name Guanchao Feng \email guanchao.feng@stonybrook.edu \\
       \name Petar M. Djuri\'{c} \email petar.djuric@stonybrook.edu \\
       \addr Department of Electrical and Computer Engineering\\
       Stony Brook University\\
       Stony Brook, NY 11794, USA}

\editor{My editor}

 \maketitle

\begin{abstract}%
The field of explainable artificial intelligence (XAI) attempts to develop methods that provide insight into how complicated machine learning methods make predictions. Many methods of explanation have focused on the concept of feature attribution, a decomposition of the model's prediction into individual contributions corresponding to each input feature.
In this work, we explore the problem of feature attribution in the context of Gaussian process regression (GPR). We take a principled approach to defining attributions under model uncertainty, extending the existing literature. We show that although GPR is a highly flexible and non-parametric approach, we can derive interpretable, closed-form expressions for the feature attributions. When using integrated gradients as an attribution method, we show that the attributions of a GPR model also follow a Gaussian process distribution, which quantifies the uncertainty in attribution arising from uncertainty in the model. We demonstrate, both through theory and experimentation, the versatility and robustness of this approach. We also show that, when applicable, the exact expressions for GPR attributions are both more accurate and less computationally expensive than the approximations currently used in practice. 
The source code for this project is freely available under MIT license at \url{https://github.com/KurtButler/2024_attributions_paper}.
\end{abstract}

\begin{keywords}
explainable ai, feature attribution, gaussian processes, integrated gradients,  sparse approximation
\end{keywords}

\tableofcontents

\section{Introduction}
With the widespread adoption of machine learning (ML) methods across academia and industry, there has been an ever increasing desire to explain how complex artificial intelligence (AI) models make predictions \cite[]{samek2021explaining, angelov2021explainable}. The ability to explain AI models, in this sense, is important to promote user trust in AI systems, and to mitigate the misuse of dubious predictions \cite[]{hess2023protoshot}. However, the process of interpreting a model to gain physical insight is often difficult unless the model was intentionally designed to be interpretable \cite[]{peng2022xai}. Interpretable models are often transparent in how predictions are made, as is the case for linear models \cite[]{arrieta2020explainable} and shallow decision trees \cite[]{silva2020optimization}, or they contain components in their architecture that are designed for interpretability, such as physics-informed neural networks \cite[]{huang2022applications}. The sentiment that many ML methods are `black boxes' reflects the inherent difficulty of interpreting complicated, data-driven models like deep neural networks and kernel machines \cite[]{zhang2018black, rudin2019stop}. Despite this, the ML paradigm has remained immensely useful, and the exploration of post-hoc explanations for black box models has been a popular subject in recent years, often referred to as explainable artificial intelligence (XAI).

A wide variety of distinct notions of explainability exist in the literature \cite[]{jin2022explainable}. Many XAI approaches can be described as being either model-specific or model agnostic. Model-specific approaches attempt to formulate explanations by exploiting the known architecture of a ML model. In the context of neural networks, this includes methods that interpret activation weights \cite[]{springenberg2014striving,bach2015pixel}, graphs \cite[]{ying2019gnnexplainer}, or attention mechanisms \cite[]{voita2019analyzing} to explain how predictions were made. Model-specific approaches also include the use of models that are inherently interpretable, including decision trees \cite[]{silva2020optimization} and kernel machines with Automatic Relevance Determination (ARD) \cite{neal1996bayesian}. 

In contrast, model-agnostic approaches seek notions of explainability that can be applied to a wide range of models, with only generic assumptions about the type of data or properties of the model. Explanations of this type might measure the sensitivity of the prediction to perturbations at the input \cite[]{baehrens2010explain, simonyan2014deep} or measuring changes in the model prediction after masking or perturbing input features \cite[]{robnik2008explaining, strumbelj2009explaining}. In prediction tasks where images are inputs, a popular tool of this type is that of saliency maps, which are images that visually highlight the relevant features in a prediction \cite[]{adebayo2018sanity}.
Another class of methods attempt to explain a prediction by showing relevant similar examples that were observed in the training set, sometimes called example-based explanations
\cite{van2021evaluating}.
Reverse image searching approaches attempt to find similar training examples to a given prediction and have been of particular interest in the field of image diffusion models \cite[]{carlini2023extracting}. 
Finally, many XAI methods can be understood through the invocation of an approximate, interpretable model, as seen in Local Interpretable Model-agnostic Explanations (LIME) \cite[]{ribeiro2016model} and SHapley Additive exPlanations (SHAP) \cite[]{lundberg2017unified}.

Many XAI methods attempt to explain a prediction by associating a numerical score with each input feature. A natural approach to assigning this score is to attempt to measure the contribution of each feature to the output, that is, how much of the resulting prediction do we \emph{attribute} to each input feature. This approach to explainability, called \emph{feature attribution}, was popularized by \cite{strumbelj2014explaining}, who introduced Shapley values to XAI. Shapley values were proposed by \cite{shapley1953value} as a method for determining what fraction of the total profit each player is owed in a coalitional game. In the context of XAI, this was reinterpreted to quantify how much each feature contributed to a numerical prediction. Since then, several other feature attribution approaches have been proposed, which we review in the next section.

In this paper, we will focus on Gaussian process regression (GPR) as a general non-parametric machine learning paradigm and assess the utility of post-hoc explanations for GPR models. This setting is intriguing because (a) interesting theoretical results are obtained from computing the attributions of a GPR model and (b) the resulting form of the GPR result is heavily dependent on the observed data set. 
Since the GPR paradigm is data-driven, any post-hoc explanation in this setting must be compatible with the lack of explicit structural knowledge in the GPR approach. Acknowledging critiques of post-hoc explanations and the XAI paradigm \cite[]{rudin2019stop, ghassemi2021false, de2022perils}, we will argue that the attributions can remain useful to characterize aspects of the learned model, as long as the practitioners' expectations are tempered to be in line with what information can actually be recovered from the attributions.

\subsection{Related Work}
Since our primary notion of explainability in this work is feature attribution, we briefly highlight existing related work in this area. The most well-known approach to feature attribution is that of Shapley values, originally developed from game theoretic principles for econometric applications by \cite{shapley1953value}. Shapley values benefit from an axiomatic definition that uniquely defines them, but this original formulation is only valid for discrete features \cite[]{aumann2015values}. In continuous domains, the original axioms are insufficient to uniquely define the method \cite[]{sundararajan2020many}. Several extensions of the Shapley value to continuous features exist, including Aumann-Shapley \cite[]{aumann2015values} and Shapley-Shubik \cite[]{friedman1999three}. Other approaches to bring Shapley values to continuous domains have attempted the inclusion/deletion of features \cite{gromping2007estimators}, or through the invocation of a proxy model that locally explains the original model, including LIME \cite[]{ribeiro2016model} and SHAP \cite[]{lundberg2017unified}. The Integrated Gradients (IG) framework of \cite{sundararajan2017axiomatic} generalizes the framework of Aumann-Shapley. \cite{sturmfels2020visualizing} and \cite{lundberg2017unified} considered variations on the IG attribution theory where baselines are given by a distribution of possible baselines. 

While the majority of research in feature attribution has concentrated on neural models, there are authors who explore Gaussian process (GP) models. \cite{chau2023explaining} proposed a GP-based version of SHAP analysis to provide uncertainty quantification of Shapley values. \cite{seitz2022gradient} considered the application of IG attributions to GP models for regression and classification. Previous research has also focused on designing GPR models to be interpretable, typically through kernel design. This can be seen in the work on ARD \cite[]{neal1996bayesian} and additive kernels \cite[]{duvenaud2011additive}. \cite{yoshikawa2021gaussian} considered a GP-based hierarchical model in which predictions are made using conditionally linear models of interpretable features, with the model weights derived from nonlinear functions of other non-interpretable features.

\subsection{Contributions} 
This paper examines several aspects of attribution theory in the context of GP models. We show that the attributions represent a decomposition of the predictive model into a sum of component functions, and this decomposition preserves the differentiability of the original model. When a GP probability distribution is used to model the function space, we show that the attribution functions are also distributed according to GPs. 

We derive exact expressions for the attributions of various common GPR models. The resulting expressions are straightforward to implement and check in code. We also comment on extensions of these results to more sophisticated kernels and GPR models. 

Using the attribution GP, we gain the ability to measure the confidence of each attribution. We demonstrate through experimental data cases where these considerations are relevant to ML practitioners in medicine and industry. 

We also compare the performance of exact attributions to various approximations that have appeared in the literature. These approximations are often asymptotic, and to the knowledge of the authors no experimental work has been done to ascertain how the convergence of these approximate solutions relates to the measured attributions.

\section{The Feature Attribution Problem}
Consider the setting of multivariable regression. Given observations $y_n$, for $n=1,...,N$, of an unknown function with corresponding input feature vectors $\x_n$ in $\mathbb{R}^D$, the goal is to estimate a function $F:\mathbb{R}^D \to \mathbb{R}$ that approximates the data, that is, $y_n \approx F(\x_n)$. Once we have learned a particular model $F$ for the data, we can make new predictions, and we may be interested to know how each of the $D$ input features contribute to a given prediction. For example, if the function $F$ is supposed to represent a clinical prediction, such as the time needed to recovery from an injury, then we may be interested to know what features in the prediction model are the most important for determining the patient's outcome. The goal is that by understanding the relative importance of features used for prediction, we may be able to explain how the model leverages this information to make predictions. 

As stated, the explainability problem is still vague, and there have been various interpretations of this idea that make the problem more concrete.
In the {\sl baseline attribution problem} \cite[]{lundstrom2022rigorous}, one `explains' a prediction $F(\x)$ at a location $\x$ by measuring how it differs from the prediction at a baseline location $\xb$ (a reference point) and allocating a fraction of the observed difference to each input feature $x_i$. Generally, the attributions to each feature should add up to the difference between the prediction $F(\x)$ and the reference prediction $F(\xb)$, which is called the \emph{completeness} property. Many feature attribution approaches use a similar approach, most notably approaches based on Shapley values \cite[]{shapley1953value}. We will focus on the baseline attribution problem in this paper. 

\textbf{Notation:}
Let $F:\mathbb{R}^D \to \mathbb{R}$ be a function representing a regression model, $\x \in \mathbb{R}^D$ be a feature vector from which we are interested in making a prediction, and $\xb \in \mathbb{R}^D$ be a reference point. 
The attribution of the prediction to the feature $x_i$ is denoted by $attr_i(\x|F)$. When the function $F$ is understood, we may suppress it from the notation and write simply $attr_i(\x)$. 
Using this notation, the completeness  property can be written as
\begin{equation}
    \label{eq:complete}
F(\x)-F(\xb) = \sum_{i=1}^D attr_i(\x|F).
\end{equation}

We sometimes have the need to compute the attributions of a vector-valued function $\mathbf{F}(\x) = [F_1(\x) \cdots F_Q(\x)]^\top$. In this case, we define the attribution vectors $\mathbf{attr}_i(\x|\mathbf{F})$ by computing the attribution for each component of $\mathbf{F}$:
\begin{equation}
    \label{eq:vectorattr}
\mathbf{attr}_i(\x|\mathbf{F}) \stackrel{\Delta}{=} \begin{bmatrix}
    attr_i(\x|F_1) \\ \vdots \\ attr_i(\x|F_Q)
\end{bmatrix}, \qquad i = 1,...,D.
\end{equation}

Finally, we use the notation $C^k (\mathcal{X})$ to denote the set of all $k$-times continuously differentiable functions $\mathcal{X} \to \mathbb{R}$, where $k=1,2,...$ or $k=\infty$. We may write $C^k$ when the functions' domain is clear from context.

\subsection{Motivation: Bayesian Linear Regression}
\label{sec:bayeslinreg}
To motivate our work in this paper, it will be instructive to briefly consider the problem of Bayesian linear regression \cite[sec.11.7]{murphy2022probabilistic}. Our reasons are twofold: Firstly, linear models provide a canonical notion of attribution which is generalized by the IG method, and second, it explicates how uncertainty in the model gives rise to uncertainty in the attributions.

Consider the typical Bayesian linear regression model, in which predictions $y$ at a point $\x$ are given a model of the form 
$$
y|\x, \mathbf{w} \sim \mathcal{N}(\mathbf{w}^\top \x, \sigma^2),
$$
where the noise (or residual) is tacitly assumed to be zero-mean, and $\sigma^2$ is its  variance. Predictions from this model are made using the expected value, so our predictive function is simply $F(\x) = \mean(y|\x,\mathbf{w})= \mathbf{w}^\top \x$. After choosing an appropriate baseline point $\xb$, the difference of a given prediction to the baseline can always be expressed as a sum over the feature index $i$,
\begin{equation}
    \label{eq:lincomplete}
    \notag
F(\x) - F(\xb) = \mathbf{w}^\top (\x - \xb)
= \sum_{i=1}^D w_i (x_i - \tilde{x}_i).
\end{equation}
In this way, the quantity $attr_i(\x) = w_i(x_i-\tilde{x}_i)$ is the canonical choice for the attribution to the $x_i$ feature. Attributions of this form satisfy the completeness property, and are straightforward to interpret. The challenge of attribution methods for nonlinear models is to generalize this notion in a manner that is consistent with the linear case. 

Now let us consider the situation in which the model parameters $\mathbf{w}$ are random variables, and not fixed unknown parameters. In this way, attributions in a Bayesian setting should account for the uncertainty in the model. The proper way to do this is to take the attributions above, and to use the posterior distribution over $\mathbf{w}$ to quantify our uncertainty in $attr_i(\x)$. We now do this for a simple example. 

Suppose that we have a data set $\mathcal{D} = \{ (\x_n,y_n) | n= 1,...,N \}$, which we compactly represent in matrix form:
$$
\mathbf{X} \stackrel{\Delta}{=} \begin{bmatrix}
    \x_1^\top \\ \vdots \\ \x_N^\top
\end{bmatrix} \in \mathbb{R}^{N \times D},
\qquad 
\mathbf{y} \stackrel{\Delta}{=} \begin{bmatrix}
    y_1 \\ \vdots \\ y_N
\end{bmatrix} \in \mathbb{R}^{N}.
$$
If we assume a Gaussian prior for the model weights,
\begin{align}
    \mathbf{w} \sim \mathcal{N}( \boldsymbol{\mu}, \boldsymbol{\Sigma}),
\end{align}
for some fixed hyperparameters $\boldsymbol{\mu}, \boldsymbol{\Sigma}$, then the posterior distribution for $\mathbf{w}$ is also Gaussian and given by,
\begin{align}
    \mathbf{w}|\mathbf{X},\mathbf{y} &\sim \mathcal{N}( \boldsymbol{\mu}', \boldsymbol{\Sigma}'), \\
    \boldsymbol{\Sigma}' &=  \left(\boldsymbol{\Sigma}^{-1}  + \frac{1}{\sigma^2} \mathbf{X}^\top \mathbf{X}\right)^{-1}, \\
    \boldsymbol{\mu}' &= \boldsymbol{\Sigma}' \left(\boldsymbol{\Sigma}^{-1} \boldsymbol{\mu} + \frac{1}{\sigma^2} \mathbf{X}^\top \mathbf{y}\right).
\end{align}

With the posterior distribution in mind, and observing that $attr_i(\x) = w_i (x_i - \tilde{x}_i)$ is just a scaled copy of $w_i$, it becomes easy to see thet the marginal distribution for each attribution is given by
\begin{align}
    \label{eq:attrbayeslin}
    attr_i(\x) | \mathbf{X},\mathbf{y} &\sim \mathcal{N}( \mu_i' (x_i-\tilde{x}_i),  \Sigma_{ii}'(x_i -\tilde{x}_i)^2).
\end{align}
From this expression, we can immediately observe two important facts: the variance or uncertainty in the attribution comes from uncertainty in the model, and this uncertainty scales based on the distance from the baseline in the $x_i$ coordinate. In Section \ref{sec:thm}, we obtain similar results for the GPR model, which is a Bayesian non-parametric regression method. We note that the GPR case parallels the basic results shown here.

\subsection{Integrated Gradients}
Consider a function $F:\mathbb{R}^D \to \mathbb{R}$ which we use as a predictive model and for which we want to compute the attributions. The Integrated Gradient (IG) method defines the \emph{attribution} of the prediction to the $i$th input feature, given a specific input $\x$, to be
\begin{equation}
    \label{eq:IG}
    attr_i(\x|F) \stackrel{\Delta}{=} (x_i-\tilde{x}_i) \int_0^1 \frac{\partial F (\xb + t(\x-\xb))}{\partial x_i} dt.
\end{equation}
The expression in \eqref{eq:IG} can be motivated from two perspectives, one using geometric ideas and the other using axioms. Both perspectives are useful for gaining insight into how we interpret IG attributions.

\subsubsection{Geometric Interpretation}
The fundamental theorem of line integrals states that the line integral of a gradient equals the difference of the function evaluated at the endpoints \cite[p. 291]{lee2012smooth}. For a continuously differentiable function $F:\mathbb{R}^D \to \mathbb{R}$, 
$$
\int_\Gamma \nabla F \cdot d\x 
= F(\bgamma(1)) - F(\bgamma(0)),
$$
where 
$$
\nabla = \begin{bmatrix}
    \frac{\partial}{\partial x_1}
    & \cdots &
    \frac{\partial}{\partial x_D}
\end{bmatrix}^\top
$$
is the gradient operator, and $\Gamma$ is a path in $\mathbb{R}^D$ parameterized by $\bgamma(t)$ for $0\leq t \leq 1$. The left-hand side of the equation can be expanded to be expressed in terms of the partial derivatives, which reveals that the integral may be expressed as a sum of path integrals. In particular, if we define the path  $\bgamma(t) = \xb + t(\x-\xb)$, then we may write
\begin{align}
    \label{eq:gradsum}
\int_\Gamma \nabla F \cdot d\x 
= \int_0^1 \sum_{i=1}^D \frac{\partial F(\bgamma(t))}{\partial \gamma_i} \frac{d \gamma_i}{dt} dt
 &= \int_0^1 \sum_{i=1}^D \frac{\partial F(\bgamma(t))}{\partial \gamma_i}  (x_i - \tilde{x}_i)  dt
\\& = \sum_{i=1}^D  (x_i - \tilde{x}_i) \int_0^1  \frac{\partial F(\bgamma(t))}{\partial x_i}  dt.
\label{eq:ohlook}
\end{align}
The terms in the sum in \eqref{eq:ohlook} are precisely the IG attributions in \eqref{eq:IG}. This perspective allows us to see that the completeness property, \eqref{eq:complete}, is automatically satisfied for any attributions that are defined by a path, although we must pick a specific path to yield the IG attributions.
The integral $\int_0^1  \partial F/\partial x_i  dt$ measures the average value of $\partial F/\partial x_i$ between the baseline $\xb$ and the prediction at $\x$, and acts as a nonlinear analogue to the weights of a linear model.

If the function $F$ has constant partial derivatives, $\partial F/\partial x_i = w_i$, for each $i$, then $F$ is a linear function. In this case, $attr_i(\x) = (x_i-\tilde{x}_i) \int_0^1 \partial F/\partial x_i dt = w_i (x_i-\tilde{x}_i) $, which agrees with what we defined for the linear model.  For nonlinear functions $F$, the IG attributions generalize the linear case by considering the average value of the gradient between the baseline and the new prediction.

Since the fundamental theorem of line integrals is not exclusive to differentiable models, but also applies to any continuous, piecewise-differentiable model, one can apply IG attributions to a much wider family of models and maintain the same geometric motivation. This includes ReLU networks \cite[]{glorot2011deep, Krizhevsky2012ImageNet},  splines \cite[]{wahba1990spline}, and GPs with deep kernel learning \cite[]{wilson2016deep}.

\subsubsection{Axiomatic Interpretation}
In the geometric approach to defining the IG attributions, we arbitrarily chose $\bgamma(t)$ to be the straight-line path $\xb + t(\x-\xb)$. In principle, any path $\bgamma(t)$ that connects $\x$ and $\xb$ could define a set of attributions, but a different choice of path can yield different values for the attributions. To this end, \cite{sundararajan2017axiomatic} and \cite{lundstrom2022rigorous} investigated the uniqueness of IG attributions. Building upon earlier work by \cite{friedman2004paths}, they proved that IG attributions, with $\bgamma(t)=\xb + t(\x-\xb)$, are uniquely determined by the seven axioms listed below. 

We note that several of these axioms are related to those that define Shapley values \cite[]{mitchell2022sampling}. However, Shapley values were originally defined for discrete covariates, and additional axioms are required to uniquely extend Shapley values to the continuous setting. If a different set of axioms are selected, then the resulting attribution operator may be different and represent a different attribution theory. \cite{sundararajan2020many} further discuss how different axioms induce different theories of feature attribution.
\begin{enumerate}
    \item \textbf{Sensitivity(a).} If $\x$ and $\xb$ differ in one coordinate ($x_i\neq \tilde{x}_i$ and $x_j = \tilde{x}_j$ for $j\neq i$), and also $F(\x)\neq F(\xb)$, then  $attr_i(\x|F) \neq 0$. 
    \item \textbf{Sensitivity(b)}, also known as Dummy \cite[]{friedman2004paths}.  If $F$ is a constant function with respect to $x_i$, then $attr_i(\x|F)=0$. In particular, this means that if $\partial F/\partial x_i = 0$ everywhere, then the attribution is zero. 
    \item \textbf{Implementation invariance.}
    The attributions should depend only on the mathematical properties of the predictive model, and not on a particular architecture or implementation.
    If $F,G$ are two models such that for all $\x$, $F(\x)=G(\x)$, then
    $ attr_i(\x|F) = attr_i(\x|G).$
    \item \textbf{Linearity}. If $F$ can be expressed as a weighted sum,  $F(\x) = \sum_{m=1}^M w_m F_m(\x)$, then the attributions of $F$ are also a weighted sum:
    $$
    attr_i(\x|F) = 
    attr_i\left(\x\bigg| \sum_{m=1}^M w_m F_m \right) = \sum_{m=1}^M w_m attr_i(\x|F_m).
    $$
    \item \textbf{Completeness}. The attributions sum to the difference between the prediction and baseline,
    $$
    \sum_{i=1}^D  attr_i(\x|F) = F(\x) - F(\xb).
    $$
    \item \textbf{Symmetry-preserving}. For a fixed pair $(i,j)$, define $\sigma_{ij}(\x)$ to be $\x$ but with the entries $x_i$ and $x_j$ swapped. We say that $x_i$ and $x_j$ are \emph{symmetric w.r.t. a function $F$} if $F(\sigma_{ij}(\x))=F(\x)$.
    An attribution method is \emph{symmetry-preserving} if
    $$
    attr_i(\x|F) = attr_j(\x|F),
    $$
    whenever $x_i=x_j$, $\tilde{x}_i=\tilde{x}_j$, and $x_i$ and $x_j$ are symmetric w.r.t. $F$.
    \item \textbf{Non-decreasing positivity}. A path $\bgamma(t)$ in $\mathbb{R}^D$ is called \emph{monotone} for each $i=1,...,D$, we have that either $d \gamma_i(t)/dt > 0$ or $d \gamma_i(t)/dt < 0$ for all $t$. \emph{Non-decreasing positivity} asserts that if $F$ is non-decreasing along \emph{every} monontone path from $\xb$ to $\x$, then $attr_i(\x|F) \geq 0$ for all $i=1,...,D$.
\end{enumerate}
Given these seven axioms, uniqueness can be claimed when  $ F \in C^1(\mathbb{R}^D)$,
that is, when $F$ is continuously differentiable \cite[]{lundstrom2022rigorous}. 
Whether or not a machine learning model satisfies requirements like $F\in C^1$ depends on the model architecture.
The differentiability requirement is often automatically satisfied when using GPR, which we discuss in Section \ref{sec:gpcalculus}.

\subsection{Linear Operators}
\label{sec:operators}
The basic motivation for feature attribution so far has focused on the case in which we have a fixed input $\x$ from which we want to make a prediction. However, as we vary the point $\x$ continuously throughout the input space, the attributions $attr_i(\x)$ also change continuously. For this reason, it may be fruitful to not only consider $attr_i(\x|F)$ as a fixed quantity, but rather as a function of the input $\x$.

We define the \textbf{attribution operator} $attr_i$ to be the linear operator that takes a base function $F$ and produces the attribution function $attr_i(\x|F)$. In this way, completeness can be viewed as a decomposition of the predictive function $F$ into components that correspond to the attributions to each feature. The attribution operators also preserve smoothness, since $attr_i(\x|F) \in C^\infty(\mathbb{R}^D)$ whenever $F \in C^\infty(\mathbb{R}^D)$ holds.\footnote{To see this, notice that IG attributions in \eqref{eq:IG} can be defined in 4 steps: taking a derivative, precomposing with a path $\bgamma(t)=\xb + t(\x-\xb)$, integrating against $t$, and then scaling by $x_i-\tilde{x}_i$. All of these operations preserve the $C^\infty$ property.} Changing the baseline $\xb$ will yield a different set of attribution operators, but they share the same properties regardless of the particular choice of $\xb$. For this reason, the baseline $\xb$ may be considered as a hyperparameter.

In Section \ref{sec:thm}, linearity of the attribution operator is the main property needed to derive results for GPR models. 
While we only analyze IG attributions in detail, the operator theoretic perspective clarifies how our results can be extended to other attribution theories.

\section{Gaussian Process Regression}
\label{sec:GPR}
In this section, we briefly review the basics of GPR and establish our notation. The basic GPR model assumes that the data to be predicted $y_n, n=1,...,N$, can be represented by an unknown function $F$ of known input vectors $\x_n$ for each point, with a constant additive noise:
\begin{equation}
\label{eq:nam}
    y_n = F(\x_n) + \varepsilon_n, \qquad \varepsilon_n \stackrel{iid}{\sim} \mathcal{N}(0,\sigma^2),
    \qquad n=1,...,N.
\end{equation}

The goal of GPR is to infer the function $F$ in a non-parametric and Bayesian manner, given the set of training data $\{(\x_n,y_n); n=1,...,N\}$.
To learn the function, one must first put a prior distribution over $F$. Generally, the prior is used to encode qualitative properties of the function such as continuity, differentiability or periodicity \cite[]{rasmussen2006gaussian}. In GPR, the prior for $F$ is taken to be a \emph{Gaussian process} with mean function $m$ and covariance function $k$, denoted as
$$
F \sim \gp(m,k).
$$
The function $F({\x})$ is a GP if for any finite set of inputs $\x_n$, the function's values $F(\x_n)$ have a jointly Gaussian distribution specified by $m$ and $k$, that is,
\begin{equation}
\label{eq:gpprior}
    \begin{bmatrix}
        F(\x_1)
        \\ \vdots
        \\ F(\x_N)
    \end{bmatrix}
    \sim
    \mathcal{N} \left( 
    \begin{bmatrix}
        m(\x_1)
        \\ \vdots
        \\ m(\x_N)
    \end{bmatrix}
    ,
    \begin{bmatrix}
        k(\x_1,\x_1) & \cdots & k(\x_1,\x_N)
        \\ \vdots & & \vdots
        \\ k(\x_N,\x_1)& \cdots & k(\x_N,\x_N) 
    \end{bmatrix}\right).
\end{equation}

It is conventional to introduce some notation here. We write 
$$
\mathbf{m} \stackrel{\Delta}{=}
    \begin{bmatrix}
        m(\x_1)
        \\ \vdots
        \\ m(\x_N)
    \end{bmatrix},
\qquad 
\mathbf{K} \stackrel{\Delta}{=} \begin{bmatrix}
        k(\x_1,\x_1) & \cdots & k(\x_1,\x_N)
        \\ \vdots & & \vdots
        \\ k(\x_N,\x_1)& \cdots & k(\x_N,\x_N) 
    \end{bmatrix}, 
    \qquad
\mathbf{k}(\x) \stackrel{\Delta}{=}
    \begin{bmatrix}
        k(\x_1,\x)
        \\ \vdots
        \\ k(\x_N,\x)
    \end{bmatrix},
$$
to make the equations more concise. 

To learn functions using GPR, we consider what happens when we augment \eqref{eq:gpprior} with a new data point $\x_*$, separate from the data used during training, and we attempt to predict $F(\x_*)$, the value of the function at this location. From \eqref{eq:nam} and \eqref{eq:gpprior}, we may write that
\begin{equation}
\label{eq:gppriorpost}
\begin{bmatrix}
    \mathbf{y} \\ F(\x_*)
\end{bmatrix}
\sim
\mathcal{N} \left(
\begin{bmatrix}
    \mathbf{m} \\ m(\x_*)
\end{bmatrix}
,
\begin{bmatrix}
    \mathbf{K} +  \sigma^2 \mathbf{I}_{N} & \mathbf{k}(\x_*) 
    \\ \mathbf{k}(\x_*)^\top & k(\x_*,\x_*) 
\end{bmatrix}
\right),
\end{equation}
where $\y = [y_1 \cdots y_N]^\top$ and $\mathbf{I}_{N}$ is the $N\times N$ identity matrix.  Since the values of $\y$ were already observed, we can condition on this data to obtain a posterior probability distribution for $F(\x_*)$, which is again Gaussian, %
\begin{align}
\label{eq:gpposterior}
    F(\x_*) | \y &\sim \mathcal{N}\left(
    \mu_*, \sigma^2_* \right),
    \\ \mu_* &=  m(\x_*) +  \mathbf{k}(\x_*)^\top (\mathbf{K} + \sigma^2 \mathbf{I}_N)^{-1} (\y - \mathbf{m}), 
\label{eq:gppostmean}
    \\ \sigma^2_* &= k(\x_*, \x_*) - \mathbf{k}(\x_*)^\top (\mathbf{K} +  \sigma^2 \mathbf{I})^{-1} \mathbf{k}(\x_*). 
\label{eq:gppostvar}
\end{align}
We omit the derivation of these expressions, as they can be found in several standard references \cite[p. 84]{murphy2022probabilistic}.

Viewing \eqref{eq:gpposterior} as a function of $\x_*$, we gain the ability to reason about the function $F$ in a Bayesian manner. 
An important fact about these expressions is that the posterior expected value of the function, $\mean(F(\x_*)|\y)$, can be expressed as a weighted sum of kernel functions,
\begin{equation}
\mean(F(\x_*)|\y) = \mathbf{k}(\x_*)^\top (\mathbf{K} + \sigma^2 \mathbf{I}_N)^{-1}(\y-\mathbf{m}) = \sum_{n=1}^N k(\x_*,\x_n) \alpha_n,
\end{equation}
where 
\begin{equation}
\label{eq:representer}
\begin{bmatrix}
    \alpha_1 \\ \vdots \\ \alpha_N 
\end{bmatrix}
\stackrel{\Delta}{=}
(\mathbf{K} + \sigma^2 \mathbf{I}_N)^{-1}(\y-\mathbf{m}).
\end{equation}
The weighted sum in \eqref{eq:representer} is in particular useful to analyze the GPR model, allowing us to perform   calculations that would otherwise be intractable.

\subsection{Covariance Functions}
\label{sec:kernelfcn}
In many applications, the default kernel used in GPR is the squared exponential (SE) kernel \cite[]{rasmussen2006gaussian}, given by 
\begin{equation}
    k_\text{SE}(\x,\x') = \sigma^2_0 \exp \left( -\sum_{i=1}^D \frac{(x_i-x_i')^2}{2\ell^2} \right),
\end{equation}
where $\ell$ and $\sigma^2_0$ are hyperparameters. A minor modification that improves the fitting power of the SE kernel is to let each input dimension $x_i$ have its own length-scale parameter $\ell_i$, in which case we have the automatic relevance detection squared exponential (ARD-SE) kernel \cite[]{neal1996bayesian}:
\begin{equation}
\label{eq:ardse}
    k_\text{ARD-SE}(\x,\x') = \sigma^2_0 \exp \left( -\sum_{i=1}^D \frac{(x_i-x_i')^2}{2\ell_i^2} \right).
\end{equation}

\subsection{Derivatives and Integrals of GPs}
\label{sec:gpcalculus}
For computing the IG attributions, we will need to integrate and differentiate GPs. 
We already remarked that functions sampled from a GP will be differentiable when the GP mean and kernel functions are also sufficiently differentiable \cite[]{adler2007random}. In this way, it also makes sense to ask what is the probability distribution of $\partial F/\partial x_i$ when $F\sim \gp$. The answer is well known; $\partial F/\partial x_i$ is also distributed according to a GP:

\begin{lemma}[Derivatives of GPs]  
\label{lemma:diff}
    If $F \sim \gp(m,k)$, where $m\in C^1(\mathbb{R}^D)$ and $k\in C^2 (\mathbb{R}^D \times \mathbb{R}^D)$, then 
    $$
    \frac{\partial F}{\partial x_i} \sim \gp \left( \frac{\partial m}{\partial x_i} , \frac{\partial^2 k}{\partial x_i \partial x_i'}\right).
    $$
\end{lemma}
This result is well-established in the literature \cite[]{solak2002derivative,le2016brownian}, but for clarity we derive it in Appendix \ref{app:proofsgp}.

In a similar spirit, the integral of a GP is also Gaussian, but we should be careful to clarify what type of integration we are considering. In one-dimension, the indefinite integral of a GP requires taking a limit that may not be well defined, depending on the mean and covariance functions of the GP. However, we can define the definite integral fairly easily, which turns out to be a non-stationary GP. 
\begin{lemma}[Integral of a GP]
\label{lemma:integral}
    Suppose $f \sim \gp(m,k)$, where $m\in C^0(\mathbb{R})$ and $k\in C^0(\mathbb{R}\times \mathbb{R})$, and let $a\in\mathbb{R}$ be fixed. The definite integral $F(x)=\int_a^x f(t) dt$ is a GP in $x$, 
    $$
    F(x)= \int_a^x f(t) dt 
    \sim \gp(M,K),
    $$
    where
    \begin{align*}
        M(x) &= \int_a^x m(t) dt,
        \\ K(x,x') &= \int_a^x \int_a^{x'} k(s,t)ds dt. 
    \end{align*}
\end{lemma}
The integral of a GP could be established more generally, but the one-dimensional case is all we require in the current work. Non-stationarity arises from the fact that when $x=a$, the integral is exactly 0, but otherwise the integral can accumulate variance as the distance between $x$ and $a$ increases. As before, we prove this lemma in Appendix \ref{app:proofsgp}.

\subsection{Spectral Approximations to GPR}
Owing to the poor scalability of the GPR model to large numbers of training samples, several approximations are often employed in practice to apply GPR models to large data sets \cite[]{hensman2013gaussian}. In this section, we consider the class of \emph{spectral approximations} to the GPR model. 

According to Bochner's theorem, every stationary covariance function, expressed as $r(\x-\x') = k(\x,\x')$, can be represented as the Fourier transform of a finite positive measure \cite[]{stein1999interpolation}. When this positive measure has a spectral density $S$, the Wiener-Khintchin theorem \cite[Chapter 4]{rasmussen2006gaussian} states that $r$ and $S$ are Fourier duals. Since the measure represented by $S$ is finite and positive, it can be normalized to yield a probability density function $p(\mathbf{v})$.
The random feature GP (RFGP) model proposed by \cite{lazaro2010sparse} samples from this probability distribution to generate \emph{random features} used to approximate the full GP kernel. Under this approach, we approximate
$$
k(\x,\x') \approx \bphi(\x)^\top \bphi(\x'),
$$
using $M$ samples from the power spectral density of the covariance function, where
$$
\bphi(\x) = \begin{bmatrix}
    \sin(\x^\top \mathbf{v}_1) \\
    \cos(\x^\top \mathbf{v}_1) \\
    \vdots \\
    \sin(\x^\top \mathbf{v}_M) \\
    \cos(\x^\top \mathbf{v}_M)
\end{bmatrix},
$$
and
$$
\mathbf{v}_m \stackrel{\mathrm{iid}}{\sim} p(\mathbf{v}), \qquad m=1,...,M.
$$

The Wiener-Khintchin theorem is particularly useful because it allows us to explicitly compute the density $p$ in a number of cases. For example, the Fourier dual of a Gaussian function $\exp( -\x^\top \mathbf{K}^{-1} \x )$, where $\mathbf{K}$ is a positive definite matrix, is again a Gaussian. Thus, we can produce random features from SE and ARD-SE kernels by sampling from Gaussian distributions. 
Having this sparse approximation of the kernel function, we can derive the predictive model for RFGPs. From \cite{lazaro2010sparse}, the predictive model is given by 
\begin{equation}
    \label{eq:rfgpmean}
\mean(F(\x_*)) = \bphi(\x_*)^\top \mathbf{A}^{-1} \bf{\Phi}\y, 
\end{equation}
and
\begin{equation}
    \label{eq:rfgpvar}
\cov(F(\x_*),F(\x_*')) =  \sigma^2_n \bphi(\x_*)^\top \mathbf{A}^{-1} \bphi(\x_*'),
\end{equation}
where we have defined the matrices
\begin{align*}
    &\A = \bPhi \bPhi^\top + \frac{M\sigma^2_n}{\sigma^2_0} \mathbf{I}_{2M},
    & %
    \quad &\bPhi = \left[\bphi(\x_1), \bphi(\x_2), \cdots, \bphi(\x_N) \right] \in \mathbb{R}^{2M \times N},
\end{align*}
for convenience. 
According to the Wiener-Khintchin theorem, frequency vectors for an  ARD-SE kernel, \eqref{eq:ardse}, can be sampled as
$$
\mathbf{v}_m \sim \mathcal{N}\left( \mathbf{0}, \text{diag}(\boldsymbol{\ell})^{-2}  \right), \qquad m=1,...,M,
$$
where $\boldsymbol{\ell}=[\ell_1 \cdots \ell_D]^\top$.
This approach leaves us with $\dim(\x)+2$ hyperparameters in this model,
$$
\boldsymbol{\theta} = \begin{bmatrix}
    \sigma^2_n \\ \sigma^2_0 \\ \boldsymbol{\ell}
\end{bmatrix}.
$$

We close this section by noting that other variants of the RFGP exist, taking different approaches to sample random features. In general, RFGPs are not usually optimal, because the particular sample of frequency vectors $\mathbf{v}_m$ strongly determines the quality of the approximation. Additionally, other approaches may yield feature-based GP models with the same form as \eqref{eq:rfgpmean}. In particular, Hilbert space approximations to GPR have a similar form to RFGPs, but these models see improvements in performance because their features are optimally selected as opposed to being randomly sampled \cite[]{riutort2023practical}.

\section{Attributions in GPR Models}
\label{sec:thm}
In this section, we introduce our main theorem: GP models produce GP attributions. We then apply this theorem to several common GPR models to derive explicit expressions for their attributions.

\setcounter{theorem}{0}
\begin{theorem}[Integrated gradients preserve Gaussianity]
    \label{thm:gpattr}
    If $F\sim \gp(m,k)$, where we assume $m\in C^1(\mathbb{R}^D)$ and $k\in C^2(\mathbb{R}^D \times \mathbb{R}^D)$, then the IG attributions follow a  GP:
    \begin{align}
    attr_i(\x|F) \sim \gp( \mu_i, \kappa_i ), \qquad i=1,...,D,
    \end{align}
    where
    \begin{align}
    \label{eq:iggpmean}
    \mu_i(\x) &=  
    (x_i-\tilde{x}_i) \int_0^1 \frac{\partial m (\xb + t(\x-\xb))}{\partial x_i} dt,
    \\ \kappa_i(\x,\x') &= (x_i-\tilde{x}_i)(x_i'-\tilde{x}_i) \int_0^1 \int_0^1 \frac{\partial^2 k(\xb + s(\x-\xb), \xb + t(\x'-\xb))}{\partial x_i \partial x_i'} dt ds.   
    \label{eq:iggpkernel}
    \end{align}
    Furthermore, the attributions satisfy a GP version of the completeness axiom:
    \begin{equation}
        \sum_{i=1}^D attr_i(\x|F) = F(\x) - F(\xb) \sim \gp(\tilde{m},\tilde{k}), 
    \end{equation}
    where
    \begin{align}
        \tilde{m}(\x) &= m(\x)-m(\xb), \\ 
        \tilde{k}(\x,\x') &= k(\x,\x') + k(\xb,\xb) - k(\x,\xb) - k(\x',\xb).
    \end{align}
\end{theorem}

The basic motivation for this theorem is that since $attr_i$ is a linear operator, it preserves Gaussianity \cite[]{solak2002derivative}. The proof is given in Appendix \ref{app:proofthm}. Here we decompose the attribution operator into a sequence of operations through which we can compute the form of the GP.
Other attribution theories, formulated in terms of linear operators as mentioned in \ref{sec:operators}, would admit proofs in a similar manner.  It should be noted that the attribution functions are always \emph{heteroscedastic}, meaning that their variance is location-dependent. This can be seen by observing that $\kappa_i(\x,\x)$ generally grows as $x_i-\tilde{x}_i$ increases, but also that $\var(attr_i(\x)) = \kappa_i(\x,\x) =0$ whenever $x_i = \tilde{x}_i$.

Knowing the expressions for the mean and covariance of the GPR posterior predictive distribution, we may derive expressions for the attributions of a GPR model. For a bit of notation, recall from $\eqref{eq:vectorattr}$ that the attribution vector $\mathbf{attr}_i(\x|\mathbf{k}(\x))$ is a vector-valued function whose entries are the attribution operator applied to each entry of $\mathbf{k}(\x)$, that is,
    $$
    \mathbf{attr}_i(\x|\mathbf{k}) = \begin{bmatrix}
        attr_i(\x|k(\x,\x_1)) \\ \vdots \\ 
        attr_i(\x|k(\x,\x_N))
    \end{bmatrix}.
    $$
Using this notation, we present the attributions of the GPR predictive posterior model in \eqref{eq:gpposterior} via the following corollary.

\setcounter{theorem}{0}
\begin{coro}[Attributions of the GPR Posterior-Predictive]
\label{corollary:gpposterattr}
    If $F$ is distributed according to the GP posterior-predictive distribution \eqref{eq:gpposterior}, \eqref{eq:gppostmean}, and \eqref{eq:gppostvar}, then the IG attributions $attr_i(\x)$ are GPs such that
    \begin{equation}
    \mean(attr_i(\x)) = attr_i(\x|m) +  \mathbf{attr}_i(\x | \mathbf{k})^\top (\mathbf{K} + \sigma^2 \mathbf{I}_N)^{-1} (\y-\mathbf{m}),
    \end{equation}
    and
    \begin{align}
    \label{eq:attrcov}
    \cov(attr_i(\x),attr_i(\x'))  = (x_i - \tilde{x}_i) & (x_i'-\tilde{x}_i)\int_0^1 \int_0^1 \frac{\partial^2 k(\xb + s(\x-\xb), \xb + t(\x'-\xb))}{\partial x_i \partial x_i'} ds dt 
    \\ &- \mathbf{attr}_i(\x | \mathbf{k})^\top (\mathbf{K} +  \sigma^2 \mathbf{I})^{-1} \mathbf{attr}_i(\x'|\mathbf{k}), 
    \notag
    \end{align}
\end{coro}

The tractability of the expressions in the corollary depends precisely on the choice of mean and covariance functions used in the GP prior. In the following subsections, we compute these results for some noteworthy GPR models.

\subsection{SE and ARD-SE Kernels}
The attributions for the GPR model can be derived in closed form when we assume an ARD-SE kernel \eqref{eq:ardse}, although the derivation can be somewhat tedious. Assuming a zero mean prior over $F$ with ARD-SE kernel, the attributions of the GPR model are given by
\begin{align}
    \mean(attr_i(\x)) &= \sum_{n=1}^N \alpha_n A_{n,i}(\x),
    \\ \var(attr_i(\x)) &= B_i(\x) - \sum_{n=1}^N \sum_{m=1}^N \Lambda_{n,m} A_{n,i}(\x) A_{m,i}(\x), 
\end{align}
where we have defined
\begin{align*}
    A_{n,i}(\x) &=  \frac{e^{\frac{-a-b-c}{2}}  \left[\sqrt{2 \pi } e^{\frac{(2 a+b)^2}{8 a}} \left(\text{erf}\left(\frac{b}{2 \sqrt{2a}}\right)-\text{erf}\left(\frac{2 a+b}{2 \sqrt{2a}}\right)\right) (b d-2 a f)+4 \sqrt{a} d \left(e^{\frac{a+b}{2}}-1\right)\right]}{4 a^{3/2}},
    \\ B_i(\x) &= %
    (x_i-\tilde{x}_i)^2 
    \left( \frac{\sqrt{2 \pi } \text{erf}\left(\frac{\sqrt{a}}{\sqrt{2}}\right) (a w+v)}{a^{3/2}}-\frac{2 e^{-\frac{a}{2}} \left(e^{a/2}-1\right) (a w+2 v)}{a^2} \right),
\end{align*}
where $\text{erf}(\cdot)$ is the error function and 

\begin{align*}
    a &= (\x-\xb)^\top \mathbf{L}^{-2}  (\x-\xb),  \quad & \mathbf{L} = \text{diag}(\ell_1, \cdots, \ell_D),
    \\ b &= 2(\x-\xb)^\top \mathbf{L}^{-2}  (\xb-\x_n), \quad &v = -\sigma^2_0 \frac{ (x_i -\tilde{x}_i)^2}{\ell_i^4},
    \\ c &= (\xb-\x_n)^\top \mathbf{L}^{-2}  (\xb-\x_n),  &w = \frac{\sigma^2_0 }{\ell_i^2},
    \\ d &= -\sigma^2_0 \frac{(x_i-\tilde{x}_i)^2}{\ell_i^2},  &[\Lambda_{n,m}] = (\mathbf{K} + \sigma^2_n \mathbf{I})^{-1},
    \\ f &=  -\sigma^2_0 \frac{(x_i-\tilde{x}_i)(\tilde{x}_i-x_{n,i})}{\ell_i^2},   & [\alpha_n] = (\mathbf{K} + \sigma^2_n \mathbf{I})^{-1} \y,
\end{align*}
are constants which implicitly depend upon $\x$, $\xb$, $\x_n$ and $i$.
The derivation of these expressions is provided in Appendix \ref{app:derivations}.

\subsection{Random Feature GPs}
Based on \eqref{eq:rfgpmean}, it is easy to see that IG attributions for the RFGP model can also be derived. The mean attribution and variance can be given as
\begin{align}
    \mean(attr_i(\x)) &= (x_i-\tilde{x}_i)\boldsymbol{\zeta}(\x) \mathbf{A}^{-1} \boldsymbol{\Phi} \y, \\
    \var(attr_i(\x)) &= (x_i-\tilde{x}_i)^2 \sigma^2_n  \boldsymbol{\zeta}(\x)^\top \mathbf{A}^{-1} \boldsymbol{\zeta}(\x),
    \notag
\end{align}
where for clarity, we introduced the notation
$$
\boldsymbol{\zeta}(\x) \stackrel{\Delta}{=}
    \begin{bmatrix}
    \frac{v_{1,i} }{\mathbf{v}_1^\top (\x-\xb)} (\cos( \mathbf{v}_1^\top \x)-\cos( \mathbf{v}_1^\top \xb))\\
    \frac{v_{1,i} }{\mathbf{v}_1^\top (\x-\xb)}(\sin( \mathbf{v}_1^\top \x)-\sin( \mathbf{v}_1^\top \xb)) \\
    \vdots \\
    \frac{v_{M,i}}{\mathbf{v}_M^\top (\x-\xb)}\left( \cos( \mathbf{v}_M^\top \x)  - \cos( \mathbf{v}_M^\top \xb) \right)\\
    \frac{v_{M,i}}{\mathbf{v}_M^\top (\x-\xb)} \left( \sin( \mathbf{v}_M^\top \x) - \sin( \mathbf{v}_M^\top \xb)\right) 
    \end{bmatrix}
    \in \mathbb{R}^{2M\times 1},
$$
and
$$ 
\mathbf{A}=\boldsymbol{\Phi \Phi}^\top + \frac{M\sigma^2_n}{ \sigma^2_0} \mathbf{I}_{2M}.
$$
The derivation of these expressions is provided in Appendix \ref{app:derivations}.

\subsection{Comments on Other GP Models}
While we only provided closed-form results for the SE and ARD-SE kernels, Theorem \ref{thm:gpattr} applies generally to all GPR models. We now briefly discuss how these results extend to other GP models.

\subsubsection{Additive GPs}
In GP kernel design, it is common to add and multiply kernels to produce new composite kernels \cite[p.19]{gibbs1998bayesian}. 
If $\gp(m,k)$ and $\gp(m',k')$ are two GPs, then the direct sum GP is defined to be $\gp(m+m',k+k')$. Using Corollary \ref{corollary:gpposterattr}, analytical expressions for the direct sum GP can be obtained whenever the component kernels $k$ and $k'$ are tractable.
\cite{duvenaud2011additive} defines the \emph{full additive kernel} by summing over all possible combinations of kernels that one may obtain through including or exclusion of features. Given a choice of base kernel $k_i$ for each input feature $x_i$, the full additive kernel is defined to be
$$
k_\textrm{add}(\x,\x') = \sum_{d=1}^D \sum_{1 \leq i_1 < \cdots < i_d \leq D} \sigma^2_d  \prod_{i=i_1,...,i_d} k_i(x_i,x_i'),
$$
where $\sigma^2_d, d=1,...,D,$ are hyperparameters. Because of the linearity property of the attribution operator, attributions can be written in closed form for the full additive kernel whenever the integrals of the base kernel are tractable.

\subsubsection{Nonstationary GPs}
\emph{Nonstationary GPs} are GPs whose covariance functions are not translation-invariant, and thus may exhibit different amplitudes, length scales and other properties depending on one's location in the input space.  Several approaches to using nonstationary GPs for inference exist: scaling the kernel by parametric functions to introduce heteroscedasticity \cite[pp.16-18]{gibbs1998bayesian}, deep kernel learning \cite[]{wilson2016deep} and deep GPs \cite[]{damianou2013deep}. Theorem \ref{thm:gpattr} applies to directly nonstationary GPs with no modification.

\section{Experiments}
We now validate our theory using several real world and simulated data sets. We provide several representative examples here, and we show additional results in the appendix.

\subsection{Simulated Data: Heteroscedasticity of the Attribution GPs}
In this example, we examine how the variance of an attribution grows as we move farther away from the baseline. To visualize this in a controlled manner, we introduce a synthetic data set to validate our theory. Consider the following generative model, where inputs $X_1$ and $X_2$ combine to produce an output $Y$:
\begin{align}
    \label{eq:syntheticdatamodel}
    X_1 &\sim \mathcal{U}(0,10), \\ \notag
    X_2 &\sim \mathcal{U}(0,10), \\ \notag
    N_Y &\sim \mathcal{N}(0,1), \\ \notag
    Y &= \sin(X_1) \sin(2X_2) + \frac{1}{2} N_Y.
     \notag
\end{align}

We sample 500 input-output pairs to train a GPR model with ARD-SE kernel. The mean predictions of this model are visualized as a function of $X_1$ and $X_2$ in Figure \ref{fig:hetero}. To study attributions, we will consider the line segment parameterized by 
\begin{equation}
\bbeta(t) = 
t \begin{bmatrix}
    1 \\1 
\end{bmatrix},
\qquad 0\leq t \leq 10.
\end{equation}
As we increase $t$, $\bbeta(t)$ moves along the diagonal of the square $[0,10]\times[0,10]$. If we set $\xb=\mathbf{0}$, then the distance between $\bbeta(t)$ and $\xb$ increases with $t$, and we expect that the attributions will increase in variance as we move farther away from the baseline.

In Figure \ref{fig:hetero}, we plot $attr_i(\bbeta(t)|F)$ against $t$, for $i=1,2$. 
We observe that for both attribution functions, the variance of the attribution increases with $t$. As $t\to 0$, both the mean and variance of the attributions tend towards zero. The rate at which the variance of the attribution grows depends on the uncertainty of the GPR predictive model. 

\begin{figure}
    \centering
    \includegraphics[width=0.9\textwidth]{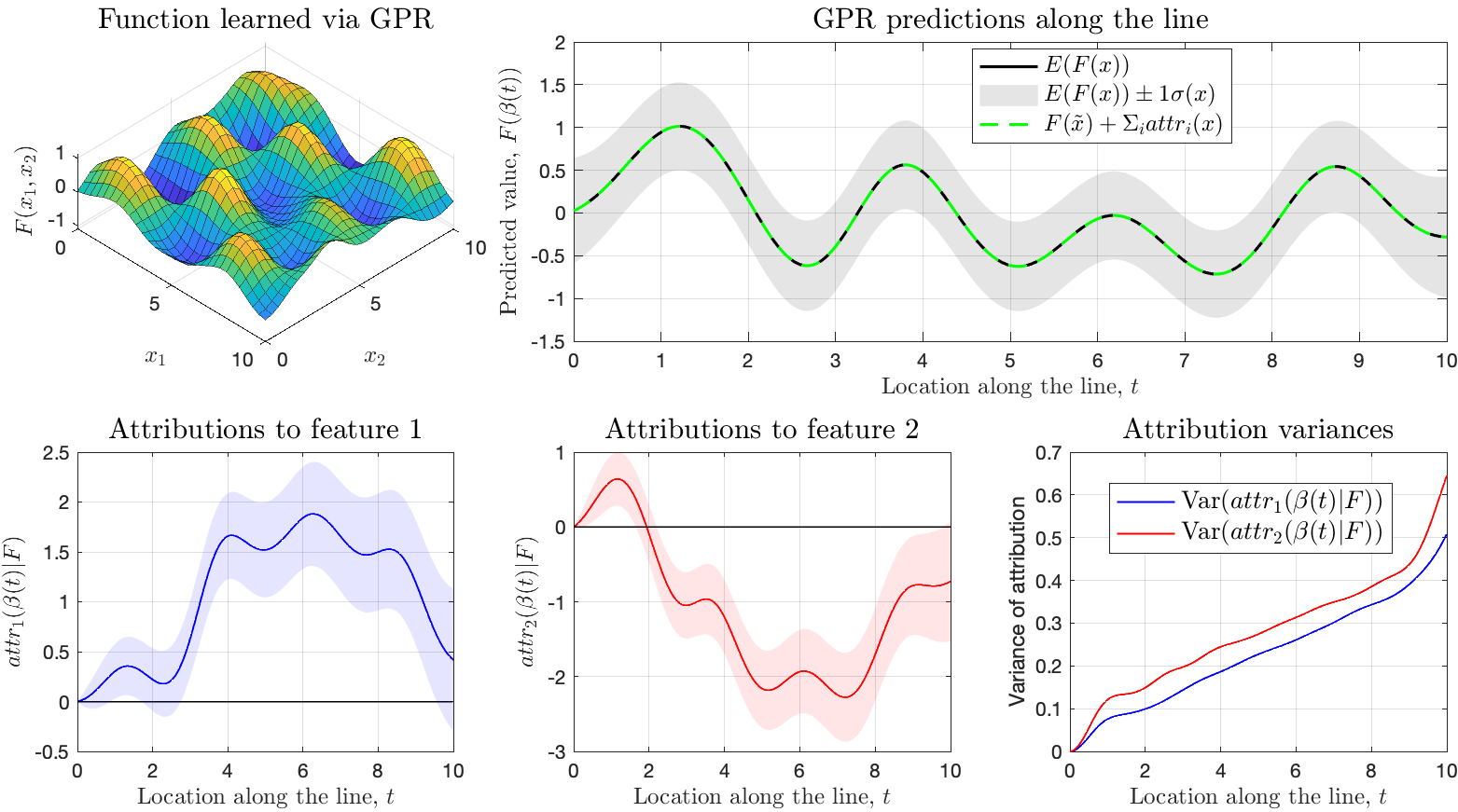}
    \caption{A demonstration that the uncertainty in an attribution grows with the distance from the baseline, using data simulated with the model in \eqref{eq:syntheticdatamodel}. The baseline is $\xb=\mathbf{0}$ and we compute attributions for predictions along the path $\bbeta(t)=t[1,1]^\top$.}
    \label{fig:hetero}
\end{figure}

\subsection{Breast Cancer Prognosis: A General Example}
In this case study, we consider the breast cancer prognostic data set of \cite{misc_breast_cancer_wisconsin_(prognostic)_16}, available through the UCI Machine Learning Repository. In this data set, there 30 features (10 measurements from 3 different recording sites) obtained from digitized images of nucleation sites in breast tissue. The goal of analyzing this data set is to predict either the recurrence time, or the disease-free time, of a patient after being screened. To this end, we trained a GPR model with an ARD-SE kernel to predict the disease-free time of patients in this data set. 
The goal of XAI in this setting is to understand, given our GPR predictive model, what features were contributing to the prediction, or in other words, which measurements at which recording sites best informed our prediction. We define the baseline $\xb$ to be average feature vector across all patients. In practice, feature attributions are sensitive to the choice of baseline point, and careful consideration is required to select a meaningful reference. As an example, we adopt the relatively simple approach of defining the baseline to be the average patient, and thus the IG method will attribute the change of a patient's disease-free time according to how an individual patient's measurements differ from the average patient.

In Figure \ref{fig:breast_cancer_a}, we visualize all attributions for a single patient. In Figure  \ref{fig:breast_cancer_b}, we visualize how two of these attributions vary across 50 different patients. Attributions in this case measure how variations in each feature contribute to the difference between the prediction for a given patient versus a reference patient with the baseline features. Thus, the value of a feature could contribute positively or negative towards this difference from the baseline. We observe that the GPR model with ARD-SE kernel prefers to attribute its predictions to only a few features, because the ARD-SE kernel automatically performs feature selection. Across patients, we observe that a feature might contribute strongly or weakly to a prediction. As the magnitude of an attribution increases in a particular feature, so too does the uncertainty in the attribution, which is something we observed in the Bayesian linear regression model of Section \ref{sec:bayeslinreg}

\begin{figure}
     \centering
     \begin{subfigure}[b]{0.9\textwidth}
         \centering
        \includegraphics[width=0.98\textwidth]{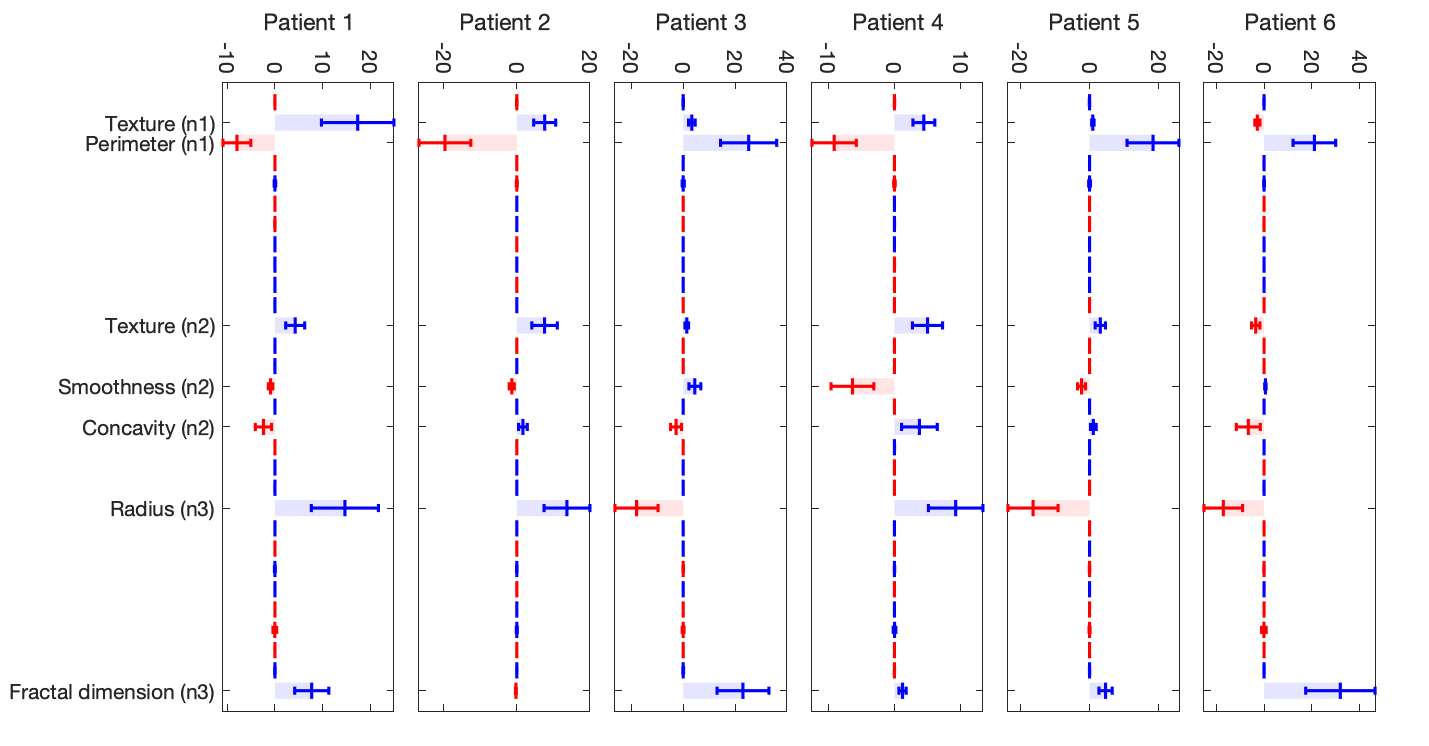}
        \caption{For 6 patients, we computed the feature attributions and their uncertainties. Feature names are only shown for consistently significant contributors to the prediction. The shaded region expresses the mean attribution and the error bars show one standard deviation in uncertainty. }
        \label{fig:breast_cancer_a}
     \end{subfigure}
     \hfill
     \begin{subfigure}[b]{0.95\textwidth}
         \centering
         \includegraphics[width=0.8\textwidth]{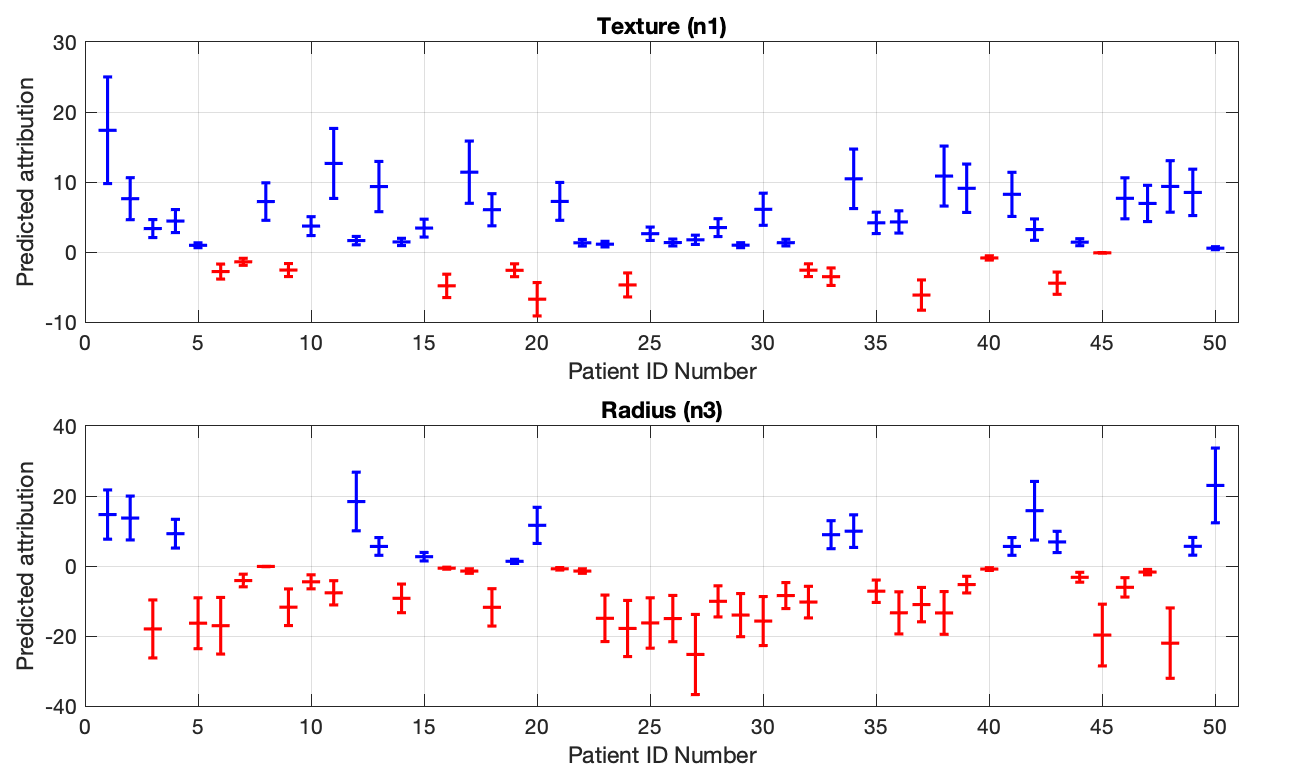}
         \caption{Two features are selected to visualize how confidence and mean attribution vary across 50 patients. 
         As in Figure 1, the midline represents the mean and the height of the cross is one standard deviation of $attr_i(\x|F)$ for each patient.}
         \label{fig:breast_cancer_b}
     \end{subfigure}
        \caption{Feature attribution for patients from the breast cancer data set. Predictions were made using a GPR model with ARD-SE kernel.}
        \label{fig:breast_cancer}
\end{figure}

\subsection{Taipei Housing Data: Demonstrating Affine Scale Invariance}
In this example, we consider real estate data in New Taipei City obtained from the UCI Machine Learning Repository \cite[]{misc_real_estate_valuation_data_set_477} in order to demonstrate a property of IG attributions called affine scale invariance (ASI). In this data set, there are 6 features used to predict the value of a house. After training a GPR model with ARD-SE kernel to predict the house value, we wish to attribute these predictions to each of the input features to understand how this information contributed to the predicted value. We take the baseline $\xb$ to be the average across all houses, and consider a prediction for a new house excluded from the training set. 

As in many ML applications, we normalized the input data before training the GPR model. However, we would like to know if this procedure has distorted our predicted attributions. Fortunately, the ASI property says that the IG attributions are invariant to scaling and translations in the input feature space \cite[]{lundstrom2022rigorous}. To formalize this statement, let us consider a mapping $T(\x) = \mathbf{\lambda}\odot \x + \mathbf{v}$, where $\odot$ represents a Hadamard (entrywise) product of vectors and all $\lambda_i \neq 0$. The mapping $T$ scales each coordinate $x_i$ by $\lambda_i$ and then translates by $v_i$. ASI asserts that
$$
attr_i(T(\x),T(\xb)|F\circ T^{-1})
= attr_i(\x,\xb|F), \qquad i=1,...,D,
$$
for any such transformation $T$. 

An important note about this result is that while the attributions do not theoretically depend upon normalization, the training of a ML model still may benefit from normalization, and this is especially true when using GPR. Thus, the ASI property suggests that normalization doesn't distort attributions by itself, but rather any change in the resulting attributions will be due to the change in how the ML model tunes it's parameters.

In Figure \ref{fig:housing}, we consider the attributions for a particular house both with and without normalization. We observe that although the input feature vectors appear noticeably different, the resulting attributions are identical. The normalized model also retains a degree of interpretability; we see that the given house has an above average distance to the MRT (subway station), and as a result this contributed negatively to the house value.

\begin{figure}
    \centering
    \includegraphics[width=0.7\textwidth]{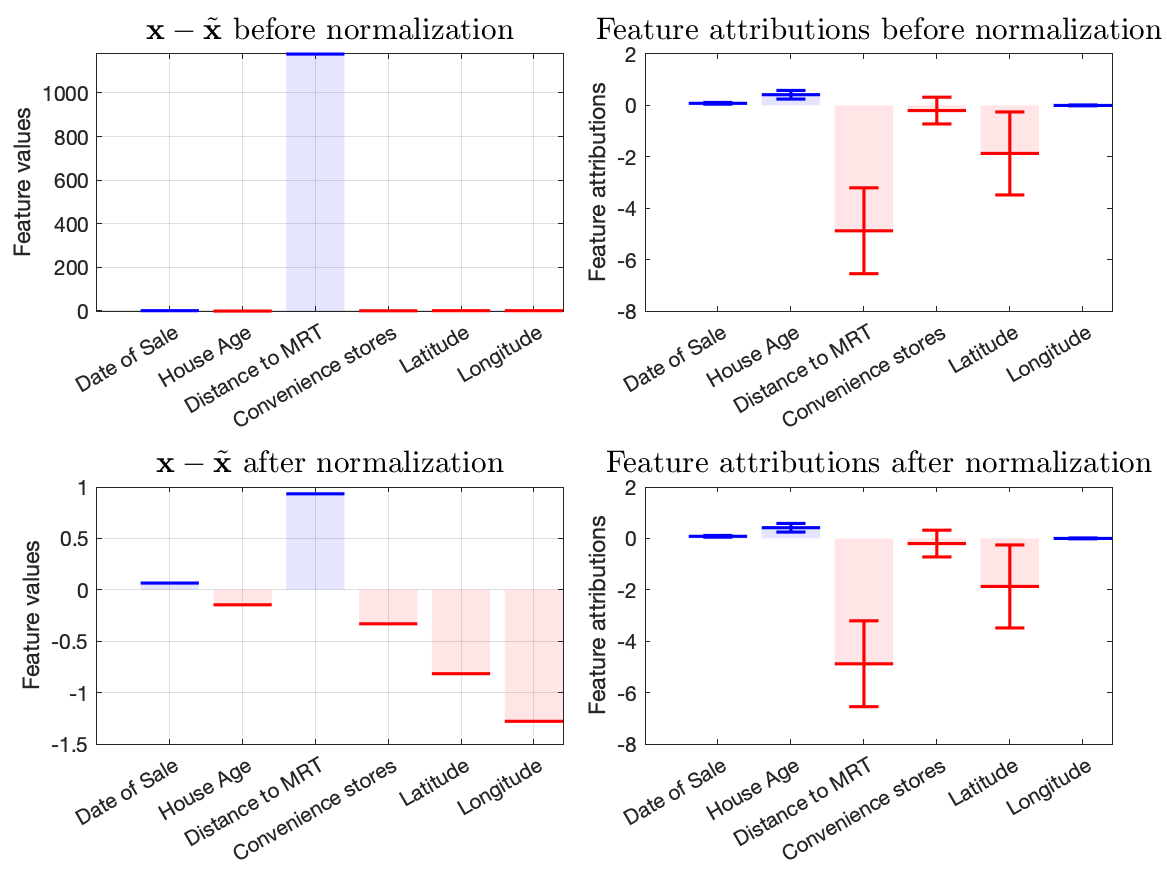}
    \caption{Feature attributions for a prediction using the Taipei Housing Data.}
    \label{fig:housing}
\end{figure}

\subsection{Wine: Approximating the Integral in IG Attributions}
In this experiment, we used the wine quality data set of \cite{Cortez2009ModelingWP}, available through the UCI Machine Learning Repository \cite[]{misc_wine_quality_186}. The goal of feature attribution with this data set is to attribute the predicted wine quality to the 11 physiochemical features that were used to make the prediction. Wines are scored on a scale of 1 to 10, with $5$ being an average quality wine. Prior to inference, the data set was randomly shuffled, the input features were normalized, and the baseline $\xb$ was chosen to be average of all wines which had average quality. We consider prediction for a particular instance $\x$, whose quality was above-average, which we omitted from the training set. 

In this experiment, we examine the effect of approximating the integral in the IG attribution formula \eqref{eq:IG}.
\cite{sundararajan2017axiomatic} do not compute this integral directly, but rather  they approximate it using a summation:
\begin{equation}
    \label{eq:approxIG}
(x_i-\tilde{x}_i) \int_0^1 \frac{\partial F(\xb + t(\x-\xb))}{\partial x_i} dt \approx \frac{(x_i-\tilde{x}_i)}{L} \sum_{l=1}^L \frac{\partial F(\xb + (l/L)(\x-\xb))}{\partial x_i}.
\end{equation}
This technique is known as the \emph{right-hand rule} for numerical integration. There exist other approaches for numerical approximation of the integral, called quadrature rules \cite[]{heath2018scientific}, which we compare to the right-hand rule in Appendix \ref{app:quadrature}.
Similar to the exact result, the approximation yields  a Gaussian distribution for the attribution \cite[]{seitz2022gradient}. This approximation closely approximates the distribution of the exact result in Theorem 1.
The quality of the approximation depends on the number of partitions $L$ used to compute the sum. 

In Figure \ref{fig:figgy4}, we visualize the convergence of the approximation to the exact GPR result as $L$ is increased. We consider the wine quality example as well as simulated data generated by  \eqref{eq:syntheticdatamodel}, to assess  how this convergence might depend on the data set and the learned model. When using the right-hand rule, $L$ is also the number of function evaluations needed to compute the approximation. In general, the exact result only needs to be evaluated once, so using a large $L$ to compute the attribution might become expensive if computed naively. How large $L$ should be to obtain a good approximation will depend on the data set and the complexity of the learned model. We also note that the approximation quality depends on the proximity to $\xb$, as points farther from the baseline will require more partitions to achieve the same resolution as points close to $\xb$.

\begin{figure}
     \centering
     \begin{subfigure}[b]{0.4\textwidth}
    \centering
    \includegraphics[width=0.99\textwidth]{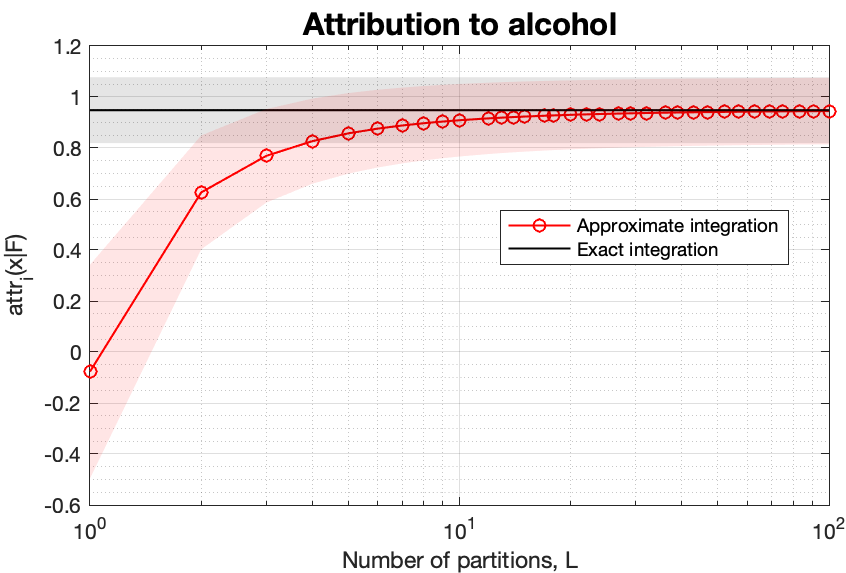}
    \caption{Results from the Wine quality data set.}
    \label{fig:approx_integral}
     \end{subfigure}
     \hfill\begin{subfigure}[b]{0.4\textwidth}
    \centering
    \includegraphics[width=0.99\textwidth]{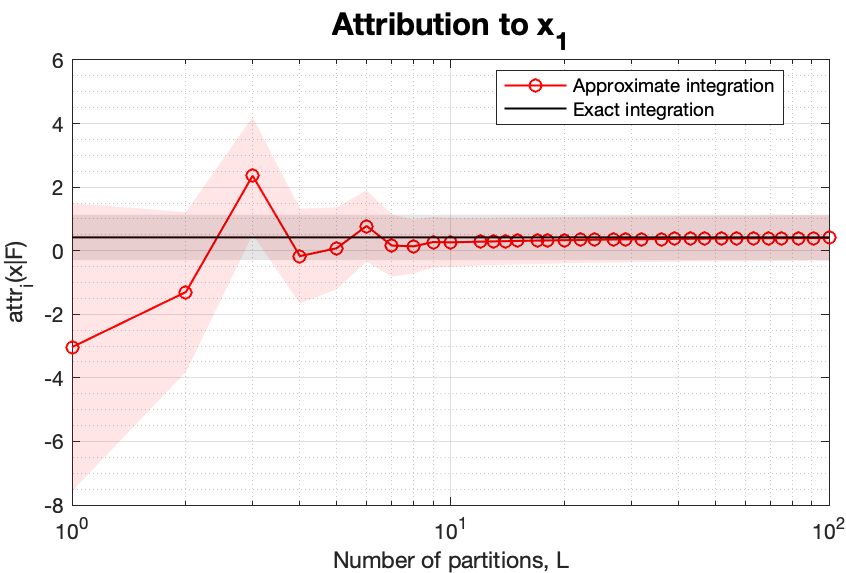}
    \caption{Results from the simulated system, \eqref{eq:syntheticdatamodel}.}
    \label{fig:approx_integral_2}
     \end{subfigure}
        \caption{Convergence of attributions with the approximated integral, \eqref{eq:approxIG}, to the exact attributions, as the number of partitions $L$ increases.}
        \label{fig:figgy4}
\end{figure}

\subsection{Wine: Comparing the Attributions of RFGPs and Exact GPR}
Much like the previous section, we used the wine quality data set to contrast an approximation with the exact attributions in Theorem 1. However, in this case we consider the RFGP approximation to the GPR kernel function.

In Figure \ref{fig:rfgp_convergence}, we consider a prediction from the wine quality data set, and we visualize the distribution of the attribution to a feature using the exact GPR result, and then how the RFGP approximation compares. We observe that as $M$ increases, the individual samples from the RFGP model become more similar, but often there is still a bias in the RFGP attributions, although this bias appears to decrease as we ad frequencies to the model. In addition to individual RFGPs, we consider a marginalized RFGP model, under which we averaged the distributions of 500 separate RFGP models.

\begin{figure}[h]
    \centering
    \includegraphics[width=0.97\textwidth]{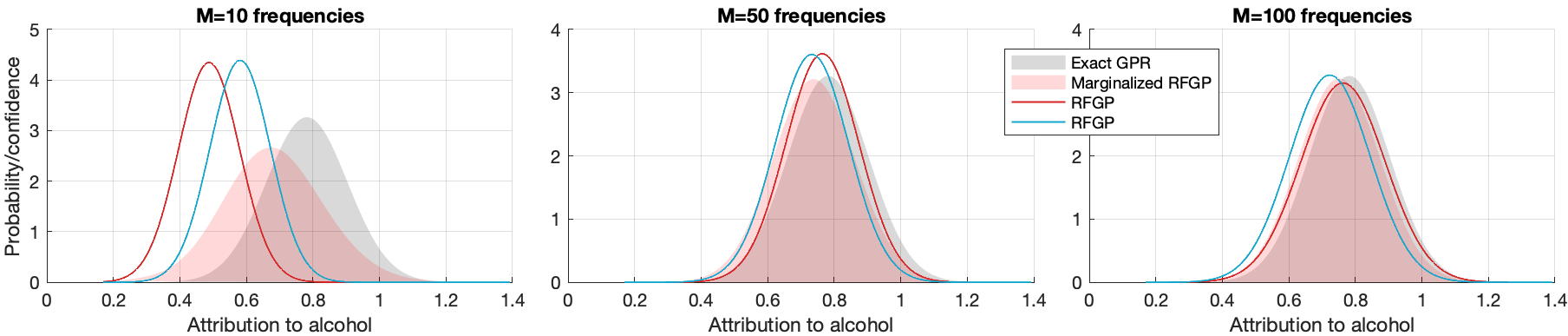}
    \caption{
    Attributions to alcohol in the Wine quality data set.
    As the number of frequencies used in the RFGP approximation, $M$, increases, the densities of the RFGP attributions more closely align with the exact GPR attributions. We show both individual realizations of the RFGP model (the blue and red curves), and a marginalized model which integrates over 500 independent RFGP models (the red shaded region).}
    \label{fig:rfgp_convergence}
\end{figure}

\section{Conclusion}
In this work, we develop a bridge between the theory of feature attribution and GPR. By representing attributions as distributions, we can quantify a model's confidence in its feature attributions. When the selected GPR model admits analytically tractable attributions, using closed-form expressions for the attributions is more computationally efficient, and it mitigates the need for computationally expensive or otherwise inaccurate approximations. While non-parametric methods like GPR, at face value, offer a non-interpretable approach to ML, our study of feature attribution may provide  more trust in GPR models. Furthermore, these results will extend the theoretical and practical scope of GPR models to more XAI domains.

\acks{This work was supported by the National Science Foundation under Awards 2021002 and 2212506.
}

\appendix
\section{Proofs}
\label{app:proofsgp}
In this appendix, we provide mathematical proofs for the lemmas and main theorem of the paper. We divide Appendix A. into four sections: In Section \ref{app:gaussops}, we introduce some additional technical lemmas.
In Section \ref{app:proofdiff}, we derive the expression for the derivatives of a GP, and in Section \ref{app:proofint}, we do the same for definite integrals of a GP.
Finally, in Section \ref{app:proofthm} we prove that the IG attributions of a GPR model are GPs.

We use the notation $C^\ell (\mathcal{X})$ to denote the set of all $\ell$-times continuously differentiable functions $\mathcal{X} \to \mathbb{R}$, where $\ell=1,2,...$ or $\ell=\infty$. We may write $C^\ell$ when the functions' domain is clear from context.

\subsection{Operations that Preserve Gaussianity}
\label{app:gaussops}
We say that a random process $(X(t))_{t\in T}$, with index set $T$, is a GP if for every finite set of indices $t_1,...,t_D$, the vector $(X(t_1),...,X(t_D))$ has a multivariate (possibly degenerate) Gaussian distribution. We use the notation $X(t) \sim \gp(m,k)$, where $m(t) = \mean(X(t))$ and $k(t,t') = \cov(X(t),X(t'))$. 
From this definition, we can immediately show that scaling and precompositions preserve Gaussianity:

\setcounter{theorem}{2}
\begin{lemma}[Scaling preserves Gaussianity]
    Let $X(t) \sim \gp(m,k)$ be a GP with index set $T$ and let $g:T \to \mathbb{R}$ be a function. 
    Then $g(t)X(t)$ is a GP, and for any finite set of indices $t_1,...,t_D$, we have that
    \begin{equation}
\label{eq:scaledgp}
    \begin{bmatrix}
        g(t_1) X(t_1)
        \\ \vdots
        \\ g(t_D) X(t_D)
    \end{bmatrix}
    \sim
    \mathcal{N} \left( 
    \begin{bmatrix}
        g(t_1) m(t_1)
        \\ \vdots
        \\ g(t_D) m(t_D)
    \end{bmatrix}
    ,
    \begin{bmatrix}
        g(t_1)^2 k(t_1,t_1) & \cdots & g(t_1)g(t_D) k(t_1,t_D)
        \\ \vdots & & \vdots
        \\ g(t_D) g(t_1) k(t_D,t_1)& \cdots & g(t_D)^2 k(t_D,t_D) 
    \end{bmatrix}\right).
    \end{equation}
\end{lemma}

\begin{proof}
    For any set of indices $t_1,...,t_D$, the pointwise scaling by $g(t)$ can be represented in matrix form, 
    $$
    \begin{bmatrix}
        g(t_1) X(t_1)
        \\ \vdots
        \\ g(t_D) X(t_D)
    \end{bmatrix}
    =
    \begin{bmatrix}
        g(t_1) & & \\
        & \ddots & \\
        & & g(t_D) 
    \end{bmatrix}
    \begin{bmatrix}
        X(t_1)
        \\ \vdots
        \\ X(t_D)
    \end{bmatrix}.
    $$
    Multiplying a Gaussian vector by a matrix preserves Gaussianity, so the GP property is upheld for the scaled process.
\end{proof}

\begin{lemma}[Precompositions preserve Gaussianity]
Let $X(t) \sim \gp(m,k)$ be a GP with index set $T$, and let $\gamma:S \to T$ be a function between sets. Then $X(\gamma(s))$ is a GP indexed by $S$, and for any finite set of indices $s_1,...,s_D$, 
\begin{equation}
\label{eq:pullbackgp}
    \begin{bmatrix}
        X(\gamma(s_1))
        \\ \vdots
        \\ X(\gamma(s_D))
    \end{bmatrix}
    \sim
    \mathcal{N} \left( 
    \begin{bmatrix}
        m(\gamma(s_1))
        \\ \vdots
        \\ m(\gamma(s_D))
    \end{bmatrix}
    ,
    \begin{bmatrix}
        k(\gamma(s_1),\gamma(s_1)) & \cdots &  k(\gamma(s_1),\gamma(s_D))
        \\ \vdots & & \vdots
        \\ k(\gamma(s_D),\gamma(s_1))& \cdots &  k(\gamma(s_D),\gamma(s_D)) 
    \end{bmatrix}\right).
    \end{equation}
\end{lemma}
\begin{proof}
    Let $t_i = \gamma(s_i)$ for $i=1,...,D$. Then $t_1,...,t_D$ is a set of indices in $T$ and hence 
    $$
    \begin{bmatrix}
        X(\gamma(s_1))
        \\ \vdots
        \\ X(\gamma(s_D))
    \end{bmatrix}
    =
    \begin{bmatrix}
        X(t_1)
        \\ \vdots
        \\ X(t_D)
    \end{bmatrix}
    \sim 
    \mathcal{N} \left( 
    \begin{bmatrix}
        m(t_1)
        \\ \vdots
        \\ m(t_D)
    \end{bmatrix}
    ,
    \begin{bmatrix}
        k(t_1,t_1) & \cdots &  k(t_1,t_D)
        \\ \vdots & & \vdots
        \\ k(t_D,t_1)& \cdots &  k(t_D,t_D) 
    \end{bmatrix}\right).
    $$
    The only thing to notice here is that if $\gamma$ is not injective (one-to-one), then the new GP is degenerate, regardless if the original GP was non-degenerate. 
\end{proof}

\subsection{Derivatives of GPs}
\label{app:proofdiff}
As noted in the main text, differentiation and integration also preserve Gaussianity. To show this, we use the following proposition from \cite[p.3]{le2016brownian}:

\setcounter{theorem}{0}
\begin{prop}
\label{prop:legall}
    Let $(X_n)_{n\geq 1}$ be a sequence of real random variables such that, for every $n\geq 1$, $X_n$ follows a $\mathcal{N}(m_n,\sigma^2_n)$ distribution. Suppose that $X_n$ converges in $L^2$ to $X$. Then:
    \begin{enumerate}
        \item The random variable $X$ follows the $\mathcal{N}(m,\sigma^2)$ distribution, where $m=\lim_{n\to \infty} m_n$ and $\sigma^2 = \lim_{n\to \infty} \sigma^2_n$.
        \item This convergence also holds in all $L^p$ spaces, $1\leq p < \infty$.
    \end{enumerate}
\end{prop}

We note that the $L^2$ convergence is automatically satisfied whenever the sequences $m_n$ and $\sigma^2_n$ converge to something reasonable (in essence, when $0<\sigma^2 < \infty$). 
Using this proposition, we can now prove results for derivatives and integrals of GPs. 

\begin{proof} \textbf{of Lemma \ref{lemma:diff}}
    To start, we recall the typical definition of a partial derivative:
    $$
    \frac{\partial F}{\partial x_i} = \lim_{h\to 0} \frac{F(\x+h\mathbf{e}_i) - F(\x)}{h},
    $$
    where $\mathbf{e}_i$ is the unit vector in the $i$-th coordinate direction. To take the limit, define a sequence $h_n=1/n$, for $n\in \mathbb{Z}_{>0}$, and consider the sequence of functions $F_n(\x) = (F(\x+h_n)-F(\x))/h_n$. We see that each $F_n$ is a GP, and thus we have a sequence of GPs, $F_n\sim \mathcal{GP}(m_{i,n},k_{i,n})$. By Proposition \ref{prop:legall}, a sequence of Gaussian variables converges to a Gaussian if the mean and variance converge. In this case,
    $$
    \lim_{n\to \infty} m_{i,n} = \lim_{n\to \infty} \mean \left( \frac{ F(\x + \frac{1}{n} \mathbf{e}_i) - F(\x)}{1/n} \right)
    = \lim_{n\to \infty} \frac{m(\x + \frac{1}{n} \mathbf{e}_i) - m(\x)}{1/n}
    = \frac{\partial m}{\partial x_i},
    $$
    which exists since $m\in C^1$. For the covariance, we have that
    $$
    k_{i,n}(\x,\x') = \text{Cov}\left( \frac{F(\x + \frac{1}{n} \mathbf{e}_i) - F(\x)}{1/n} , \frac{F(\x' + \frac{1}{n} \mathbf{e}_i) - F(\x')}{1/n}  \right),
    $$
    which can then be expressed as
    $$
    k_{i,n}(\x,\x') = \frac{k(\x+\frac{1}{n}\mathbf{e}_i, \x'+\frac{1}{n}\mathbf{e}_i) - k(\x,\x'+\frac{1}{n}\mathbf{e}_i) - k(\x+\frac{1}{n}\mathbf{e}_i,\x') + k(\x,\x')}{n^2}.
    $$
    Taking the limit then yields
    $$
    \lim_{n\to \infty} k_{i,n}(\x,\x') = \frac{\partial^2 k}{\partial x_i \partial x_i'},
    $$
    which exists because $k\in C^2$. Thus, by Proposition \ref{prop:legall} we have that $\partial F/\partial x_i$ is Gaussian whose mean and covariance are specified by the functions $\partial m/\partial x_i$ and $\partial^2 k/\partial x_i \partial x_i'$.
    To conclude the proof, notice that for any finite set of points $\x_1,...,\x_L$, the vector $[F(\x_1) \cdots F(\x_L)]$ will again be multivariate Gaussian with mean and covariance controlled by these functions. Since this property holds arbitrarily, the Kolmogorov extension theorem asserts that a GP exists with the desired mean and covariance functions.
\end{proof}

We might say that a GP is $\ell$-times continuously differentiable if the mean and covariance functions are $C^\ell$ and $C^{2\ell}$, respectively, since this ensures that functions sampled from the GP are $C^\ell$ almost-surely.
Using this notion of differentiability, we can discuss how scaling and precomposition affect differentiability.   The equations \eqref{eq:scaledgp} and \eqref{eq:pullbackgp} show that the output mean and covariance functions are $C^\ell$ whenever $g$, $\gamma$ and the original $m$ and $k$ functions are $C^\ell$. Hence, scaling and precompositions by smooth functions preserve both smoothness and Gaussianity. 
Lemma 1 shows that differentiating a $C^\ell$ GP yields a $C^{\ell-1}$ GP.

\subsection{Integrals of GPs}
\label{app:proofint}
The integration of a GP is established using similar techniques to differentiation. For computing attributions, we only need to know how to compute a one-dimensional integral, over a line segment.

\begin{proof}\textbf{of Lemma \ref{lemma:integral}}
    Since the integral is being taken over a line segment, it suffices to consider the Riemannian definition of the integral as the limit over increasingly fine partitions of the line. Hence we define
    $$
    F_n(x) = \sum_{i=0}^n  \frac{1}{n} f\left( a + \frac{i (x-a)}{n} \right),
    $$
    and $F(x)$ to be the limit of $F_n(x)$ as $n\to\infty$. Similar in the proof for derivatives of a GP, we need to compute the mean and covariance functions $F_n(x)$, show that they converge, and apply Proposition \ref{prop:legall}. For the mean, the summation becomes
    $$
    \mean (F_n(x)) = \sum_{i=0}^n \frac{\mean (f(a+ i(x-a)/n))}{n}
    = \sum_{i=0}^n \frac{m(a+i(x-a)/n)}{n},
    $$
    and since $m$ is continuous on the interval, the Riemann sum converges to the integral. Thus,
    $$
    \mean(F_n(x)) \to  \int_a^x m(t)dt = M(x),
    $$
    as $n\to \infty$, and we denote the result as a function $M$ of $x$.
    For the covariance, we have that 
    \begin{align*}
    \text{Cov}\left( F_n(x), F_n(x') \right) 
    &= 
    \text{Cov}\left( \sum_{i=0}^n \frac{f(a+ i(x-a)/n)}{n}, \sum_{j=0}^n \frac{f(a+ j(x'-a)/n)}{n}\right)
    \\ &= \sum_{i=0}^n \sum_{j=0}^n \frac{1}{n^2} \cov \left( f\left( a + \frac{i(x-a)}{n}\right) , f\left(a+ \frac{j(x'-a)}{n}\right) \right)
    \\ &= \sum_{i=0}^n \sum_{j=0}^n \frac{1}{n^2} k\left( a + \frac{i(x-a)}{n} , a+ \frac{j(x'-a)}{n} \right),
    \end{align*}
    which converges, again because $k$ is continuous, to the double integral, which we define to be $K$, i.e.,
    $$
    \lim_{n \to \infty} \text{Cov}\left( F_n(x), F_n(x') \right) 
    =
    \int_a^x \int_a^{x'} k(s,t) ds dt = K(x,x'),
    $$
    as $n \to \infty$. Since the mean and covariance functions converge, we can conclude the convergence of $F_n(x) \to F(x)$ for a given point $x$. 
    As in the proof of Lemma 1, for any finite set of points $x_1,...,x_L$, we want to show that the joint distribution of $F(x_1),...,F(x_L)$ coincides with the Gaussian distribution specified by the functions $M$ and $K$. Notice that for fixed $n$, 
    $$
    \mathbf{F}_n(x_1,...,x_L) \stackrel{\Delta}{=} 
    \sum_{i=1}^n \frac{1}{n} \begin{bmatrix}
        f\left(a + \frac{i(x_1-a)}{n} \right) \\
        \vdots \\
        f\left(a + \frac{i(x_L-a)}{n} \right)
    \end{bmatrix}
    =
    \begin{bmatrix}
         \sum_{i=1}^n \frac{1}{n}  f\left(a + \frac{i(x_1-a)}{n} \right) \\
        \vdots \\
        \sum_{i=1}^n \frac{1}{n}  f\left(a + \frac{i(x_L-a)}{n} \right)
    \end{bmatrix}
    = 
    \begin{bmatrix}
         F_n(x_1) \\ \vdots \\ F_n(x_L)
    \end{bmatrix},
    $$
    and thus
    $$
    \lim_{n \to \infty} \mathbf{F}_n (x_1,...,x_L) =
    \lim_{n \to \infty} 
    \begin{bmatrix}
         F_n(x_1) \\ \vdots \\ F_n(x_L)
    \end{bmatrix} =
    \begin{bmatrix}
         F(x_1) \\ \vdots \\ F(x_L)
    \end{bmatrix}.
    $$
    From the above calculations, we then know that $\lim_{n \to \infty} \mathbf{F}_n (x_1, ... ,x_L)$ obeys a multivariate Gaussian law with mean and covariance determined by $M$ and $K$, respectively. By the Kolmogorov extension theorem, there exists a GP determined by these functions.%
\end{proof}

\subsection{Attributions of a GP}
\label{app:proofthm}
Given the previous lemmas, we are now equipped to compute the IG attributions of a GPR model and prove the theorem in Section \ref{sec:thm}. Let $F(\x)$ be a GP over $\mathbb{R}^D$. We fix a baseline point $\xb$ and recall the definition of the $attr_i$ operator:
    $$
    attr_i(\x|F) = (x_i-\tilde{x}_i) \int_0^1 \frac{\partial F(\xb + t(\x-\xb))}{\partial x_i} dt.
    $$
    The attribution is defined by four sequential operations,
    \begin{enumerate}
    \item Taking the partial derivative $\partial/\partial x_i$ of $F$,
        \item Precomposition of $\partial F/\partial x_i$ with the function $\bgamma(t) = \xb + t(\x-\xb)$,
        \item Integration from $t=0$ to $t=1$, and 
        \item Scaling the result by $(x_i-\tilde{x})$. 
    \end{enumerate}
    Together, we observe that these operations all preserve Gaussianity, and thus to find the attribution GP it suffices to compute its mean and covariance, whenever they exist. As a simple condition to ensure existence, we require that $m\in C^1(\mathbb{R}^D)$ and $k\in C^2(\mathbb{R}^D,\mathbb{R}^D)$, which guarantees that we can differentiate the GP and later integrate it.

\begin{proof}\textbf{of Theorem \ref{thm:gpattr}}
    The existence of the attribution GP follows directly from the lemmas.
     We show the computation for the reader's convenience, implicitly invoking Lemmas 1, 2, 3, and 4:
    \begin{align*}
        \mean(F(\x)) &= m(\x), & \\
        \mean\left( \frac{\partial F(\x)}{\partial x_i}\right) &= \frac{\partial m(\x)}{\partial x_i}, \\
        \mean\left( \frac{\partial F(\bgamma(t))}{\partial x_i}\right) &= \frac{\partial m(\bgamma(t))}{\partial x_i},  \\
        \mean\left( \int_0^1 \frac{\partial F(\bgamma(t))}{\partial x_i} dt \right) &= \int_0^1 \frac{\partial m(\xb + t(\x-\xb))}{\partial x_i} dt, \\
        \mean(attr_i(\x|F)) &= 
        \mean\left( (x_i-\tilde{x}_i) \int_0^1 \frac{\partial F(\bgamma(t))}{\partial x_i} dt \right) \\ &= (x_i-\tilde{x}_i) \int_0^1 \frac{\partial m(\xb + t(\x-\xb))}{\partial x_i} dt.
    \end{align*}
    The resulting mean attribution function is defined since $m \in C^1$ implies that $\partial m/\partial x_i$ is $C^0$. Since the derivative is continuous, it is also integrable over the interval $[0,1]$. 
    The calculation for the covariance function is more tedious, but follows the same logic. Let us define $\bbeta(\x,t) = \xb + t(\x-\xb)$ for conciseness. We find that
    \begin{align*}
        \cov(F(\x),F(\x')) &= k(\x,\x'), \\
        \cov\left(\frac{\partial F(\x)}{\partial x_i} , \frac{\partial F(\x')}{\partial x_i} \right) &= \frac{\partial^2 k(\x,\x')}{\partial x_i \partial x_i'}, \\
        \cov\left(\frac{\partial F(\bbeta(\x,s))}{\partial x_i} , \frac{\partial F(\bbeta(\x',t))}{\partial x_i}  \right) &= \frac{\partial^2 k(\bbeta(\x,s),\bbeta(\x',t))}{\partial x_i \partial x_i'}, \\
        \cov\left( \int_0^1 \frac{\partial F(\bbeta(\x,s))}{\partial x_i} ds, \int_0^1 \frac{\partial F(\bbeta(\x',t))}{\partial x_i}  dt \right) &=  \int_0^1 \int_0^1 \frac{\partial^2 k(\bbeta(\x,s),\bbeta(\x',t))}{\partial x_i \partial x_i'} ds dt,
    \end{align*}
    and finally, 
    \begin{align*}
        \cov(attr_i(\x),attr_i(\x')) &=
        \cov\left( (x_i-\tilde{x}_i) \int_0^1 \frac{\partial F(\bbeta(\x,s))}{\partial x_i} ds, (x_i'-\tilde{x}_i) \int_0^1 \frac{\partial F(\bbeta(\x',t))}{\partial x_i}  dt \right)
        \\ &= (x_i-\tilde{x}_i)(x_i'-\tilde{x}_i) \cov\left( \int_0^1  \frac{\partial F(\bbeta(\x,s))}{\partial x_i} ds, \int_0^1 \frac{\partial F(\bbeta(\x',t))}{\partial x_i}  dt \right)
        \\ &= (x_i-\tilde{x}_i)(x_i'-\tilde{x}_i)  \int_0^1 \int_0^1 \frac{\partial^2 k(\bbeta(\x,s),\bbeta(\x',t))}{\partial x_i \partial x_i'} ds dt.
    \end{align*}
    Again, to see that this expression is defined, we recall that $k\in C^2$, and therefore $\partial^2 k/\partial x_i \partial x_i' \in C^0$, and thus is integrable over the unit square.
    
    To show completeness of the output GP, we need to know that almost-surely every sample from a differentiable GP is differentiable. Let $f$ be any differentiable sample from the GP. The attributions $attr_i(\x|f)$ satisfy completeness by definition, and thus $\sum_i attr_i(\x|f) = f(\x) - f(\xb)$. Since this holds for almost-every $f$, the distributions of the left and right hand sides coincide. We can see from either side that the distribution is Gaussian, and we can compute that
    $$
    \mean(F(\x)-F(\xb)) = m(\x) - m(\xb),
    $$
    and
    \begin{align*}
    \cov(F(\x)-F(\xb),F(\x')-F(\xb))
    &=  \cov(F(\x),F(\x'))
    +  \cov(F(\xb),F(\xb))
    \\&\quad-  \cov(F(\x),F(\xb)) - \cov(F(\xb),F(\x'))
    \\&= k(\x,\x')  + k(\xb,\xb) 
    \\& \quad- k(\x,\xb) - k(\x',\xb).
    \end{align*}
\end{proof}

If $F\sim \gp(m,k)$ is the GPR posterior predictive distribution, then $attr_i(\x|F)\sim \gp$ whenever $m\in C^1$ and $k\in C^2$. If we satisfy the more refined assumption that $m\in C^\infty$ and $k\in C^\infty$, then we have that $attr_i(\x|F) \in C^\infty$ as well.

\section{Derivations for Specific Models}
\label{app:derivations}
In this section, we derive the expressions for the mean and variance of the attributions for specific models. 

\subsection*{ARD-SE Kernel}
To derive the GPR attributions when using an ARD-SE kernel, it is instructive to work with Corollary 1. From here, we observe that one must derive expressions for $\mathbf{attr}_i(\x|\mathbf{k})$ and the double integral in \eqref{eq:attrcov}. In the main text, we defined 
$$
A_{n,i}(\x) \stackrel{\Delta}{=} attr_i(\x|k(\cdot,\x_n)),
$$
which are the entries of the vector $\mathbf{attr}_i(\x|\mathbf{k})=[A_{n,i}]_{n=1:N}$. To compute $A_{n,i}$, we must apply the $attr_i$ operator to the kernel function $k$, which we know to be the ARD-SE function,
$$
k_\text{ARD-SE}(\x,\x') = \sigma^2_f \exp \left( -\sum_{i=1}^D \frac{(x_i-x_i')^2}{2\ell_i^2} \right).
$$
If we let $\x'=\x_n$ and $\mathbf{L}=\text{diag}([\ell_1 \cdots \ell_D])$, then we can write out the kernel function in matrix form:
$$
k(\x,\x_n) = \sigma^2_f \exp \left( -\frac{1}{2} (\x-\x_n)^\top \mathbf{L}^{-2} (\x-\x_n) \right).
$$
The partial derivative of this is
$$
\frac{\partial k(\x,\x_n)}{\partial x_i}
=
 \sigma^2_f \exp \left( -\frac{1}{2} (\x-\x_n)^\top \mathbf{L}^{-2} (\x-\x_n) \right) \left(-\frac{x_i - x_{n,i}}{\ell_i^2}\right).
$$
The expression begins to look hairy when we precompose with $\bgamma(t) = \xb + t(\x-\xb)$:
\begin{align}
\frac{\partial k(\bgamma(t),\x_n)}{\partial x_i}
&= 
 \sigma^2_f \exp{\left( -\frac{1}{2} (\xb + t(\x-\xb)-\x_n)^\top \mathbf{L}^{-2} (\xb + t(\x-\xb)-\x_n) \right)}\notag\\
 &\times \left(-\frac{\tilde{x}_i + t(x_i-\tilde{x}_i) - x_{n,i}}{\ell_i^2}\right).\notag
\end{align}
At this point, the most challenging task is to compute the integral from $t=0$ to $t=1$. The trick is to notice that the expression above is of the form
$$
\exp(-at^2 -bt - c)(d\times t + f)
$$
as a function of $t$, where $a,b,c,d$ and $f$ are constants that we later plug in. From here, one can produce closed-form expressions for the integral. It can be checked that
$$
\int  t \exp(-at^2-bt-c) dt
= -\frac{e^{-at^2 -bt -c}}{2 a} - 
\frac{\sqrt{\pi} e^{(b^2/4a)-c} \erf \left(\frac{2at+b}{2\sqrt{a}}\right) }{4a^{3/2}},
$$
and
$$
\int  \exp(-at^2-bt-c) dt
=  - 
\frac{\sqrt{\pi} e^{(b^2/4a)-c} \erf \left(\frac{2at+b}{2\sqrt{a}}\right)}{2 \sqrt{a}},
$$
where $\erf$ is the error function,
$$
\erf(t) = \frac{2}{\sqrt{\pi}} \int_0^z e^{-t^2} dt.
$$
From the indefinite integral, we can compute the definite integral by evaluating the expressions at $t=0$ and $t=1$. Based on the expression above, the appropriate definitions for the constants $a,b,c,d$ and $f$ are
\begin{align*}
    a &= (\x-\xb)^\top \mathbf{L}^{-2}  (\x-\xb),
    \\ b &= 2(\x-\xb)^\top \mathbf{L}^{-2}  (\xb-\x_n),
    \\ c &= (\xb-\x_n)^\top \mathbf{L}^{-2}  (\xb-\x_n),
    \\ d &= -\sigma^2_0 \frac{(x_i-\tilde{x}_i)^2}{\ell_i^2},  
    \\ f &=  -\sigma^2_0 \frac{(x_i-\tilde{x}_i)(\tilde{x}_i-x_{n,i})}{\ell_i^2}.
\end{align*}
After some algebra we deduce that
$$
A_{n,i}(\x) =  \frac{e^{\frac{-a-b-c}{2}} \left[\sqrt{2 \pi } e^{\frac{(2 a+b)^2}{8 a}} \left(\text{erf}\left(\frac{b}{2 \sqrt{2a}}\right)-\text{erf}\left(\frac{2 a+b}{2 \sqrt{2a}}\right)\right) (b d-2 a f)+4 \sqrt{a} d \left(e^{\frac{a+b}{2}}-1\right)\right]}{4 a^{3/2}},
$$
and thus
$$
\mean(attr_i(\x)) = \mathbf{attr}_i(\x|\mathbf{k})^\top \boldsymbol{\alpha} = \sum_{n=1}^N \alpha_n A_{n,i}(\x).
$$
Computation of the variance is also tractable, but requires some more work. Let $\mathbf{\Lambda} = (\mathbf{K} + \sigma^2_n \mathbf{I})^{-1}$. Then we observe that
$$
\var(attr_i(\x)) = B_i(\x) - \mathbf{attr}_i(\x|\mathbf{k})^\top \mathbf{\Lambda} \mathbf{attr}_i(\x|\mathbf{k}),
$$
where we have defined 
$$
B_i(\x) \stackrel{\Delta}{=} (x_i-\tilde{x}_i)^2 \int_0^1 \int_0^1 
\frac{\partial^2 k(\bgamma(s),\bgamma(t))}{\partial x_i \partial x_i'} ds dt.
$$
We already know $\mathbf{attr}_i(\x|\mathbf{k})$ from computing the mean attribution, and all that remains is to compute $B_i(\x)$. Like before, our goal is to convert $B_i(\x)$ into a generic expression and then use a known result to solve the integral. Firstly, we compute the second derivative, which we may express as 
\begin{align*}
\frac{\partial^2 k(\x,\x')}{\partial x_i \partial x_i'}
&=
\sigma^2_0  
\frac{\partial^2}{\partial x_i \partial x_i'}
\exp \left(-\frac{1}{2} \sum_{i=1}^D \frac{(x_i-x_i')^2}{\ell_i^2}\right)
\\ & = \sigma^2_0 \exp \left(-\frac{1}{2} \sum_{i=1}^D \frac{(x_i-x_i')^2}{\ell_i^2}\right) \left( \frac{1}{\ell_i^2} - \frac{(x_i-x_i')^2}{\ell_i^4} \right)
\\ & = \sigma^2_0 \exp \left(-\frac{(\x-\x')^\top \mathbf{L}^{-2} (\x-\x')}{2} \right) \left( \frac{1}{\ell_i^2} - \frac{(x_i-x_i')^2}{\ell_i^4} \right).
\end{align*}
The next step is to precompose and integrate. When $\x\neq \x'$, the general integral is complicated. The formulas simplify considerably if we constrain our attention to the case in which $\x=\x'$, meaning that we only compute the variance of an attribution, and not the full covariance function. In this case, we let $\bgamma(t)=\xb + t(\x-\xb)$, and then precomposition yields
\begin{align*}
\frac{\partial^2 k(\bgamma(s),\bgamma(t))}{\partial x_i \partial x_i'}
& = \sigma^2_0 \exp \left(-\frac{(\bgamma(s)-\bgamma(t))^\top \mathbf{L}^{-2} (\bgamma(s)-\bgamma(t))}{2} \right) \left( \frac{1}{\ell_i^2} - \frac{(\gamma_i(s)-\gamma_i(t))^2}{\ell_i^4} \right) \\
& = \sigma^2_0 \exp \left(-\frac{(s-t)^2 (\x-\xb)^\top \mathbf{L}^{-2} (\x-\xb)}{2} \right) \left( \frac{1}{\ell_i^2} - \frac{(s-t)^2(x_i-\tilde{x}_i)^2}{\ell_i^4} \right) \\
& =  \exp \left(-\frac{a(s-t)^2}{2} \right) \left( w  + v(s-t)^2 \right),
\end{align*}
where $a= (\x-\xb)^\top \mathbf{L}^{-2} (\x-\xb)$ is the same as in the calculation for the mean, and where we additionally defined
\begin{align*}
    v &= -\sigma^2_0 \frac{(x_i-\tilde{x}_i)^2}{\ell_i^4} ,
    \\ w &= \frac{\sigma^2_0}{\ell_i^2}.
\end{align*}
In this form, the integral can be evaluated, yielding 
$$
\int_0^1 \int_0^1 e^{-a(s-t)^2/2}(w+v(s-t)^2) ds dt
=
\frac{\sqrt{2 \pi } \text{erf}\left(\frac{\sqrt{a}}{\sqrt{2}}\right) (a w+v)}{a^{3/2}}+\frac{2 e^{-\frac{a}{2}} \left(e^{a/2}-1\right) (a w+2 v)}{a^2}.
$$
We may then write
\begin{align*}
    B_i(\x) 
    =
    (x_i-\tilde{x}_i)^2 
    \left( \frac{\sqrt{2 \pi } \text{erf}\left(\frac{\sqrt{a}}{\sqrt{2}}\right) (a w+v)}{a^{3/2}}-\frac{2 e^{-\frac{a}{2}} \left(e^{a/2}-1\right) (a w+2 v)}{a^2} \right).
\end{align*}

\subsection*{Random Feature GPs}
To derive the RFGP attributions, we can use Theorem \ref{thm:gpattr} directly with 
\eqref{eq:rfgpmean} and \eqref{eq:rfgpvar}.
From \eqref{eq:rfgpmean}, we see that the RFGP mean can be expressed as a weighted sum. If we define the vector $\boldsymbol{\alpha} \stackrel{\Delta}{=} \mathbf{A}^{-1}\boldsymbol{\Phi}\mathbf{y}$, then we may write
$$
\boldsymbol{\phi}(\x)^\top \mathbf{A}^{-1}\boldsymbol{\Phi}\mathbf{y}
= \boldsymbol{\phi}(\x)^\top  \boldsymbol{\alpha}
= \sum_{m=1}^{2M} \phi_m (\x) \alpha_m,
$$
where $\alpha_m$ are constants and 
$$
\phi_m (\x) = 
\begin{cases}
			\sin(\x^\top \mathbf{v}_{(m+1)/2}), & \text{if $m$ odd}\\
            \cos(\x^\top \mathbf{v}_{m/2}), & \text{otherwise}
\end{cases}.
$$
In either case, $\phi_m(\x)$ is of the form $\sin(\x^\top \mathbf{v} )$ or $\cos(\x^\top \mathbf{v} )$, for some value of $\mathbf{v}$. By linearity, it suffices to compute the attribution for this generic form, and then plug in to find the resulting expression. 

Since $\cos(x)$ can be represented as $\sin(x+b)$ where $b=\pi/2$, we show the generic calculation for $\sin(\x^\top \mathbf{v} + b)$ that will produce the desired result for both the sine and cosine cases.
Let $\phi(\x) = \sin(\x^\top \mathbf{v} + b)$. Then
\begin{align*}
attr_i(\x|\phi) &= (x_i-\tilde{x}_i) \int_0^1 \frac{\partial \phi(\xb+t(\x-\xb))}{\partial x_i} dt
\\&= (x_i-\tilde{x}_i) \int_0^1  \cos ((\xb+t(\x-\xb))^\top \mathbf{v} + b) v_i dt
\\&= (x_i-\tilde{x}_i) \left[
\sin ((\xb+t(\x-\xb))^\top \mathbf{v} + b) 
\right|_{t=0}^1 \frac{v_i}{(\x-\xb)^\top \mathbf{v}}
\\&= (x_i-\tilde{x}_i) \left( 
\sin (\x^\top \mathbf{v} + b)
-\sin (\xb^\top \mathbf{v} + b)
\right) \frac{v_i}{(\x-\xb)^\top \mathbf{v}}.
\end{align*}
From this calculation, we can derive the expression for the vector of $\phi$ attributions:
$$
\mathbf{attr}_i(\x|\boldsymbol{\phi})
= \begin{bmatrix}
    attr_i(\x|\phi_1) \\ \vdots \\ attr_i(\x|\phi_{2M})
\end{bmatrix}
= (x_i-\tilde{x}_i) 
    \begin{bmatrix}
    \frac{v_{1,i} }{\mathbf{v}_1^\top (\x-\xb)} (\cos( \mathbf{v}_1^\top \x)-\cos( \mathbf{v}_1^\top \xb))\\
    \frac{v_{1,i} }{\mathbf{v}_1^\top (\x-\xb)}(\sin( \mathbf{v}_1^\top \x)-\sin( \mathbf{v}_1^\top \xb)) \\
    \vdots \\
    \frac{v_{M,i}}{\mathbf{v}_M^\top (\x-\xb)}\left( \cos( \mathbf{v}_M^\top \x)  - \cos( \mathbf{v}_M^\top \xb) \right)\\
    \frac{v_{M,i}}{\mathbf{v}_M^\top (\x-\xb)} \left( \sin( \mathbf{v}_M^\top \x) - \sin( \mathbf{v}_M^\top \xb)\right) 
    \end{bmatrix}
$$
and thus
$$
\mean(attr_i(\x)) = \mathbf{attr}_i(\x|\boldsymbol{\phi})^\top \boldsymbol{\alpha}.
$$
In the main text, we use the notation 
$$
\boldsymbol{\zeta}(\x) = 
    \begin{bmatrix}
    \frac{v_{1,i} }{\mathbf{v}_1^\top (\x-\xb)} (\cos( \mathbf{v}_1^\top \x)-\cos( \mathbf{v}_1^\top \xb))\\
    \frac{v_{1,i} }{\mathbf{v}_1^\top (\x-\xb)}(\sin( \mathbf{v}_1^\top \x)-\sin( \mathbf{v}_1^\top \xb)) \\
    \vdots \\
    \frac{v_{M,i}}{\mathbf{v}_M^\top (\x-\xb)}\left( \cos( \mathbf{v}_M^\top \x)  - \cos( \mathbf{v}_M^\top \xb) \right)\\
    \frac{v_{M,i}}{\mathbf{v}_M^\top (\x-\xb)} \left( \sin( \mathbf{v}_M^\top \x) - \sin( \mathbf{v}_M^\top \xb)\right) 
    \end{bmatrix}
    $$
so that we may compactly write
$$
\mean(attr_i(\x)) = (x_i-\tilde{x}_i)\boldsymbol{\zeta}^\top \boldsymbol{\alpha}
= (x_i-\tilde{x}_i)\boldsymbol{\zeta}
^\top \mathbf{A}^{-1}\boldsymbol{\Phi}\mathbf{y}.
$$
This notation is also useful for deriving the covariance, since when we write
$$
\cov(attr_i(\x),attr_i(\x'))
= (x_i-\tilde{x}_i)(x_i'-\tilde{x}_i) \int_0^1 \int_0^1 \frac{\partial^2 k(\xb + s(\x-\xb), \xb + t(\x'-\xb))}{\partial x_i \partial x_i'} dt ds,
$$
for $k(\x,\x') = \sigma^2_n \boldsymbol{\phi}(\x)^\top \mathbf{A}^{-1} \boldsymbol{\phi}(\x')$, we can split up the derivatives and integrals.
Without writing everything out, the desired result is 
$$
\cov(attr_i(\x),attr_i(\x'))
= \sigma^2_n \mathbf{attr}_i(\x|\boldsymbol{\phi})^\top \mathbf{A}^{-1}  \mathbf{attr}_i(\x'|\boldsymbol{\phi}),
$$
which immediately yields
$$
\cov(attr_i(\x),attr_i(\x'))
= \sigma^2_n (x_i-\tilde{x}_i)(x_i'-\tilde{x}_i) \boldsymbol{\zeta}(\x)^\top \mathbf{A}^{-1} \boldsymbol{\zeta}(\x').
$$

\section{Additional Experiments}
In this appendix, we include some additional experiments and validation that we omit from the main text.

\subsection{Wine: Quadrature Rules}
\label{app:quadrature}
In this section, we provide some additional figures to analyze the convergence of the approximated integral.
First, we visualize the convergence of the approximated integral to the true IG attributions, for the Wine data set in Figure \ref{fig:fig4_suppl}, as only one feature was shown in Figure \ref{fig:approx_integral}. Figure \ref{fig:fig4_suppl} shows that the rate of convergence is similar across all input features, and does not obviously correspond to the magnitude of the exact attribution. 
The ARD-SE kernel also provides a measure of feature importance, where relevance of a (normalized) feature is defined to be $r_i = 1/\ell_i$. In this case, we observe that for features with relevance, their attributions converge much more quickly. This is expected since a low feature relevance indicates that model training has preferentially underused a given feature in making predictions.

\begin{figure}[h]
    \centering\includegraphics[width=0.85\textwidth]{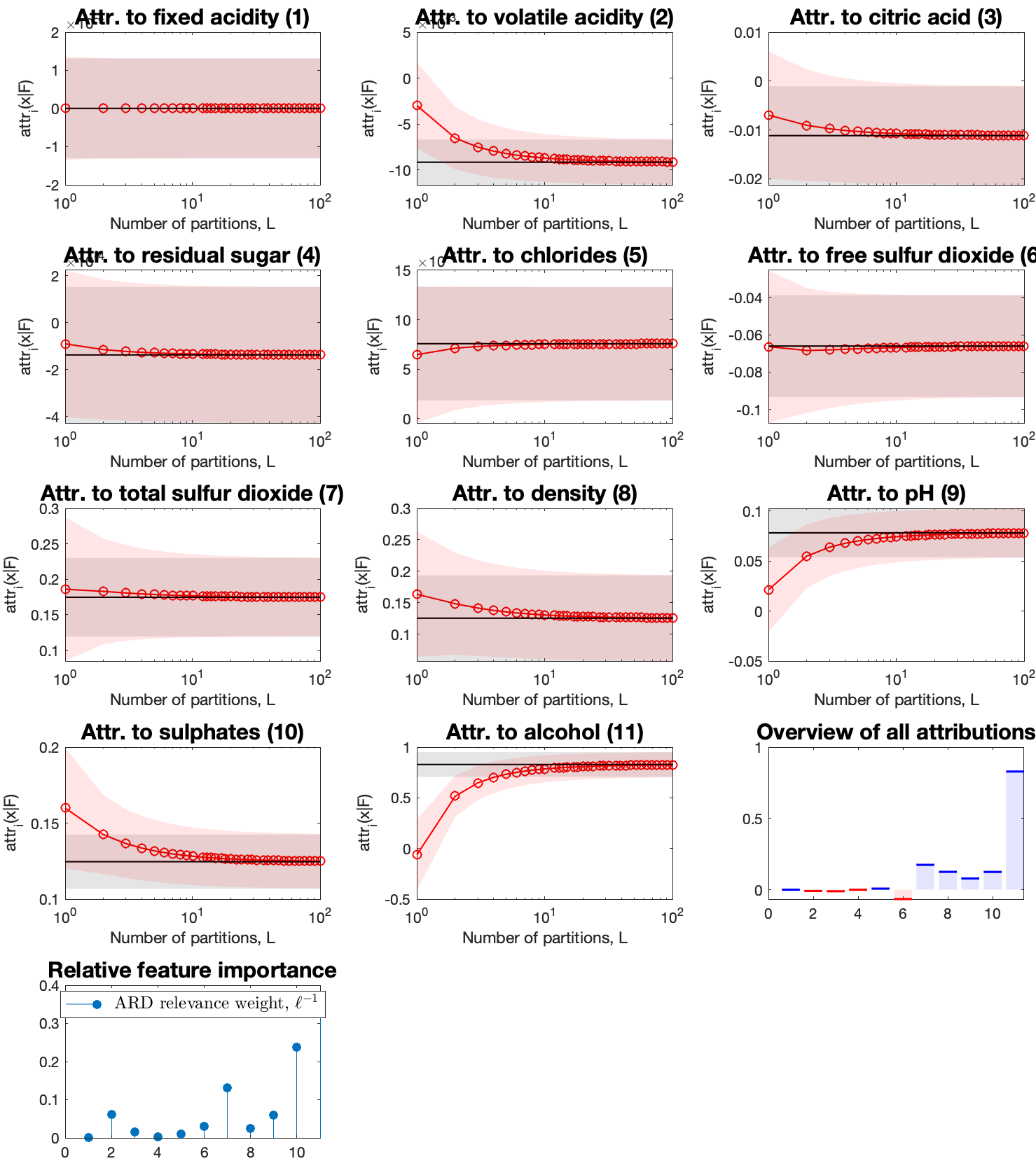}
    \caption{Convergence of the attributions with approximated integration to the exact attributions, in the Wine data set. The first 11 subfigures shown the approximation convergence, as described in Figure \ref{fig:approx_integral}. The last two subfigures show the relative sizes of each attribution in the exact result, and also the relative importance of each input feature (as measured by the ARD kernel), respectively.}
    \label{fig:fig4_suppl}
\end{figure}

Second, we visualize the impact of different quadrature rules. The formula in \eqref{eq:approxIG} uses the right-hand rule to approximate the integral of a function on the interval $[0,1]$, but there exist other quadrature rules, or methods for numerical integration, that may achieve noticeable improvements in performance with only changes to the weights in the sum used to approximate the integral \cite[]{heath2018scientific}. 
In Figure \ref{fig:three graphs}, we compare the application of the right-hand rule with two other quadrature rules: the trapezoid rule and Simpson's rule. Each method approximates the integral using a partition of the unit interval:
$$
0 = t_0 < t_1 < \cdots <t_{L-1} < t_L = 1.
$$
For each $l=1,...,L$, each quadrature approximates a portion of the integral along that segment. For a general function $g(t)$, the three quadratures use the partition to approximate the integral differently:
$$
\int_0^1 g(t) dt \approx \sum_{l=1}^L S_l
$$
\begin{align*}
    &\text{Right-hand:} & S_l = g(t_l), 
    \\
    &\text{Trapezoid:} & S_l = \frac{1}{2} \left(g(t_{l-1}) + g(t_l)\right), 
    \\
    &\text{Simpson's rule:} & S_l = \frac{1}{4} \left( g(t_{l-1}) + 2g\left(\frac{t_{l-1}+t_l}{2}\right)  + g(t_l) \right).
\end{align*}
We note that Simpson's rule includes an extra evaluation per partition. As a result, Simpson's rule should be compared to other methods relative to the number of function calls, instead of $L$.
In Figure \ref{fig:three graphs}, we compare the convergence of each quadrature rule against the number of partitions $L$ as well as the number of required function evaluations. We additionally visualize the mean-squared error (MSE) of the means of each attribution GP. Both the trapezoid rule and Simpson's rule demonstrate a marked improvement over the right-hand rule.

\begin{figure}
     \centering
     \begin{subfigure}[b]{0.32\textwidth}
         \centering
         \includegraphics[width=\textwidth]{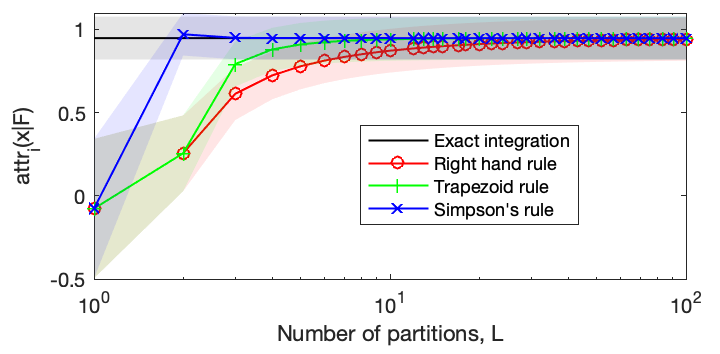}
         \caption{Attribution GPs vs the number of partitions $L$.}
         \label{fig:y equals x}
     \end{subfigure}
     \hfill
     \begin{subfigure}[b]{0.32\textwidth}
         \centering
         \includegraphics[width=\textwidth]{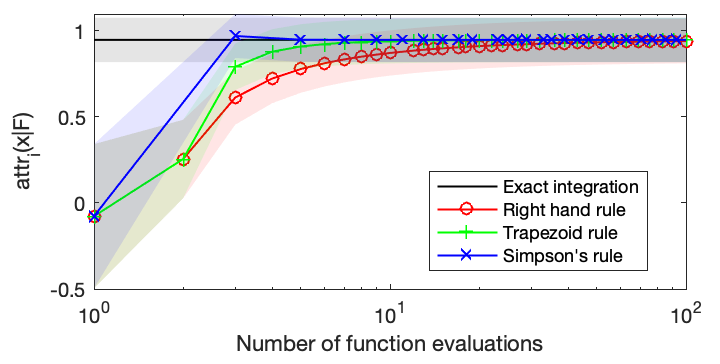}
         \caption{Attribution GPs vs the number of function evaluations.}
         \label{fig:three sin x}
     \end{subfigure}
     \hfill
     \begin{subfigure}[b]{0.32\textwidth}
         \centering
         \includegraphics[width=\textwidth]{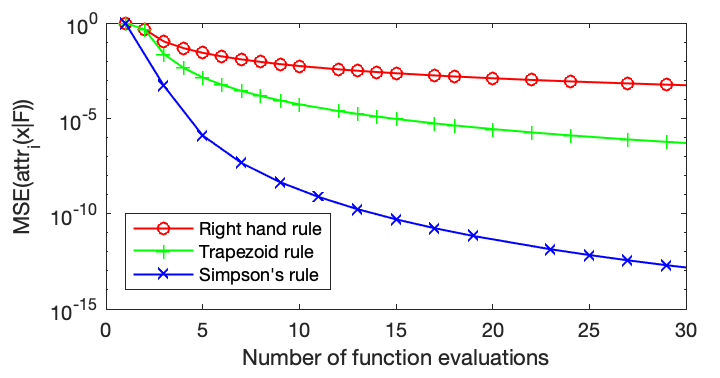}
         \caption{MSE when estimating the mean attribution.}
         \label{fig:five over x}
     \end{subfigure}
        \caption{Comparison of the three quadrature rules applied to compute the attributions of a GPR model with ARD-SE kernel. Each method was tasked with computing the attribution to Alcohol, as seen in Figure \ref{fig:fig4_suppl}. }
        \label{fig:three graphs}
\end{figure}

\vskip 0.2in
\addcontentsline{toc}{section}{References}
\bibliography{refs}

\begin{thebibliography}{62}
\providecommand{\natexlab}[1]{#1}
\providecommand{\url}[1]{\texttt{#1}}
\expandafter\ifx\csname urlstyle\endcsname\relax
  \providecommand{\doi}[1]{doi: #1}\else
  \providecommand{\doi}{doi: \begingroup \urlstyle{rm}\Url}\fi

\bibitem[mis(2018)]{misc_real_estate_valuation_data_set_477}
{Real estate valuation data set}.
\newblock UCI Machine Learning Repository, 2018.
\newblock {DOI}: https://doi.org/10.24432/C5J30W.

\bibitem[Adebayo et~al.(2018)Adebayo, Gilmer, Muelly, Goodfellow, Hardt, and
  Kim]{adebayo2018sanity}
Julius Adebayo, Justin Gilmer, Michael Muelly, Ian Goodfellow, Moritz Hardt,
  and Been Kim.
\newblock Sanity checks for saliency maps.
\newblock \emph{Advances in neural information processing systems}, 31, 2018.

\bibitem[Adler and Taylor(2007)]{adler2007random}
Robert~J Adler and Jonathan~E Taylor.
\newblock \emph{Random Fields and Geometry}.
\newblock Springer, 2007.

\bibitem[Angelov et~al.(2021)Angelov, Soares, Jiang, Arnold, and
  Atkinson]{angelov2021explainable}
Plamen~P Angelov, Eduardo~A Soares, Richard Jiang, Nicholas~I Arnold, and
  Peter~M Atkinson.
\newblock Explainable artificial intelligence: an analytical review.
\newblock \emph{Wiley Interdisciplinary Reviews: Data Mining and Knowledge
  Discovery}, 11\penalty0 (5):\penalty0 e1424, 2021.

\bibitem[Arrieta et~al.(2020)Arrieta, D{\'\i}az-Rodr{\'\i}guez, Del~Ser,
  Bennetot, Tabik, Barbado, Garc{\'\i}a, Gil-L{\'o}pez, Molina, Benjamins,
  et~al.]{arrieta2020explainable}
Alejandro~Barredo Arrieta, Natalia D{\'\i}az-Rodr{\'\i}guez, Javier Del~Ser,
  Adrien Bennetot, Siham Tabik, Alberto Barbado, Salvador Garc{\'\i}a, Sergio
  Gil-L{\'o}pez, Daniel Molina, Richard Benjamins, et~al.
\newblock Explainable artificial intelligence ({XAI}): Concepts, taxonomies,
  opportunities and challenges toward responsible {AI}.
\newblock \emph{Information Fusion}, 58:\penalty0 82--115, 2020.

\bibitem[Aumann and Shapley(2015)]{aumann2015values}
Robert~J Aumann and Lloyd~S Shapley.
\newblock \emph{Values of Non-Atomic Games}.
\newblock Princeton University Press, 2015.

\bibitem[Bach et~al.(2015)Bach, Binder, Montavon, Klauschen, M{\"u}ller, and
  Samek]{bach2015pixel}
Sebastian Bach, Alexander Binder, Gr{\'e}goire Montavon, Frederick Klauschen,
  Klaus-Robert M{\"u}ller, and Wojciech Samek.
\newblock On pixel-wise explanations for non-linear classifier decisions by
  layer-wise relevance propagation.
\newblock \emph{PLOS One}, 10\penalty0 (7):\penalty0 e0130140, 2015.

\bibitem[Baehrens et~al.(2010)Baehrens, Schroeter, Harmeling, Kawanabe, Hansen,
  and M{\"u}ller]{baehrens2010explain}
David Baehrens, Timon Schroeter, Stefan Harmeling, Motoaki Kawanabe, Katja
  Hansen, and Klaus-Robert M{\"u}ller.
\newblock How to explain individual classification decisions.
\newblock \emph{The Journal of Machine Learning Research}, 11:\penalty0
  1803--1831, 2010.

\bibitem[Carlini et~al.(2023)Carlini, Hayes, Nasr, Jagielski, Sehwag, Tramer,
  Balle, Ippolito, and Wallace]{carlini2023extracting}
Nicolas Carlini, Jamie Hayes, Milad Nasr, Matthew Jagielski, Vikash Sehwag,
  Florian Tramer, Borja Balle, Daphne Ippolito, and Eric Wallace.
\newblock Extracting training data from diffusion models.
\newblock In \emph{32nd USENIX Security Symposium (USENIX Security 23)}, pages
  5253--5270, 2023.

\bibitem[Chau et~al.(2023)Chau, Muandet, and Sejdinovic]{chau2023explaining}
Siu~Lun Chau, Krikamol Muandet, and Dino Sejdinovic.
\newblock Explaining the uncertain: Stochastic {S}hapley values for {G}aussian
  process models.
\newblock \emph{arXiv preprint arXiv:2305.15167}, 2023.

\bibitem[Cortez et~al.(2009{\natexlab{a}})Cortez, Cerdeira, Almeida, Matos, and
  Reis]{Cortez2009ModelingWP}
Paolo Cortez, Antonio~Lu{\'i}z Cerdeira, Fernando Almeida, Telmo Matos, and
  Jos{\'e} Reis.
\newblock Modeling wine preferences by data mining from physicochemical
  properties.
\newblock \emph{Decis. Support Syst.}, 47:\penalty0 547--553,
  2009{\natexlab{a}}.
\newblock URL \url{https://api.semanticscholar.org/CorpusID:2996254}.

\bibitem[Cortez et~al.(2009{\natexlab{b}})Cortez, Cerdeira, Almeida, Matos, and
  Reis]{misc_wine_quality_186}
Paulo Cortez, A.~Cerdeira, F.~Almeida, T.~Matos, and J.~Reis.
\newblock {Wine Quality}.
\newblock UCI Machine Learning Repository, 2009{\natexlab{b}}.
\newblock {DOI}: https://doi.org/10.24432/C56S3T.

\bibitem[Damianou and Lawrence(2013)]{damianou2013deep}
Andreas Damianou and Neil~D Lawrence.
\newblock Deep {G}aussian processes.
\newblock In \emph{Artificial Intelligence and Statistics}, pages 207--215.
  PMLR, 2013.

\bibitem[de~Bruijn et~al.(2022)de~Bruijn, Warnier, and Janssen]{de2022perils}
Hans de~Bruijn, Martijn Warnier, and Marijn Janssen.
\newblock The perils and pitfalls of explainable {AI}: Strategies for
  explaining algorithmic decision-making.
\newblock \emph{Government Information Quarterly}, 39\penalty0 (2):\penalty0
  101666, 2022.

\bibitem[Duvenaud et~al.(2011)Duvenaud, Nickisch, and
  Rasmussen]{duvenaud2011additive}
David~K Duvenaud, Hannes Nickisch, and Carl Rasmussen.
\newblock Additive {G}aussian processes.
\newblock \emph{Advances in Neural Information Processing Systems}, 24, 2011.

\bibitem[Friedman and Moulin(1999)]{friedman1999three}
Eric Friedman and Herve Moulin.
\newblock Three methods to share joint costs or surplus.
\newblock \emph{Journal of Economic Theory}, 87\penalty0 (2):\penalty0
  275--312, 1999.

\bibitem[Friedman(2004)]{friedman2004paths}
Eric~J Friedman.
\newblock Paths and consistency in additive cost sharing.
\newblock \emph{International Journal of Game Theory}, 32:\penalty0 501--518,
  2004.

\bibitem[Ghassemi et~al.(2021)Ghassemi, Oakden-Rayner, and
  Beam]{ghassemi2021false}
Marzyeh Ghassemi, Luke Oakden-Rayner, and Andrew~L Beam.
\newblock The false hope of current approaches to explainable artificial
  intelligence in health care.
\newblock \emph{The Lancet Digital Health}, 3\penalty0 (11):\penalty0
  e745--e750, 2021.

\bibitem[Gibbs(1998)]{gibbs1998bayesian}
Mark~N Gibbs.
\newblock \emph{Bayesian {G}aussian processes for regression and
  classification}.
\newblock PhD thesis, Citeseer, 1998.

\bibitem[Glorot et~al.(2011)Glorot, Bordes, and Bengio]{glorot2011deep}
Xavier Glorot, Antoine Bordes, and Yoshua Bengio.
\newblock Deep sparse rectifier neural networks.
\newblock In \emph{Proceedings of the Fourteenth International Conference on
  Artificial Intelligence and Statistics}, pages 315--323. JMLR Workshop and
  Conference Proceedings, 2011.
\newblock URL \url{http://proceedings.mlr.press/v15/glorot11a/glorot11a.pdf}.

\bibitem[Gr{\"o}mping(2007)]{gromping2007estimators}
Ulrike Gr{\"o}mping.
\newblock Estimators of relative importance in linear regression based on
  variance decomposition.
\newblock \emph{The American Statistician}, 61\penalty0 (2):\penalty0 139--147,
  2007.

\bibitem[Heath(2018)]{heath2018scientific}
Michael~T Heath.
\newblock \emph{Scientific Computing: An Introductory Survey}.
\newblock SIAM, second edition, 2018.

\bibitem[Hensman et~al.(2013)Hensman, Fusi, and Lawrence]{hensman2013gaussian}
James Hensman, Nicolo Fusi, and Neil~D Lawrence.
\newblock Gaussian processes for big data.
\newblock In \emph{Uncertainty in Artificial Intelligence}, page 282, 2013.

\bibitem[Hess and Ditzler(2023)]{hess2023protoshot}
Samuel Hess and Gregory Ditzler.
\newblock Protoshotxai: Using prototypical few-shot architecture for
  explainable {AI}.
\newblock \emph{Journal of Machine Learning Research}, 24\penalty0
  (325):\penalty0 1--49, 2023.
\newblock URL \url{http://jmlr.org/papers/v24/21-1261.html}.

\bibitem[Huang and Wang(2022)]{huang2022applications}
Bin Huang and Jianhui Wang.
\newblock Applications of physics-informed neural networks in power systems-a
  review.
\newblock \emph{IEEE Transactions on Power Systems}, 38\penalty0 (1):\penalty0
  572--588, 2022.

\bibitem[Jin et~al.(2022)Jin, Sergeeva, Weng, Chauhan, and
  Szolovits]{jin2022explainable}
Di~Jin, Elena Sergeeva, Wei-Hung Weng, Geeticka Chauhan, and Peter Szolovits.
\newblock Explainable deep learning in healthcare: A methodological survey from
  an attribution view.
\newblock \emph{WIREs Mechanisms of Disease}, 14\penalty0 (3):\penalty0 e1548,
  2022.

\bibitem[Krizhevsky et~al.(2012)Krizhevsky, Sutskever, and
  Hinton]{Krizhevsky2012ImageNet}
Alex Krizhevsky, Ilya Sutskever, and Geoffrey~E Hinton.
\newblock {ImageNet} classification with deep convolutional neural networks.
\newblock In \emph{Advances in Neural Information Processing Systems},
  volume~25, 2012.

\bibitem[L{\'a}zaro-Gredilla et~al.(2010)L{\'a}zaro-Gredilla,
  Quinonero-Candela, Rasmussen, and Figueiras-Vidal]{lazaro2010sparse}
Miguel L{\'a}zaro-Gredilla, Joaquin Quinonero-Candela, Carl~Edward Rasmussen,
  and An{\'\i}bal~R Figueiras-Vidal.
\newblock Sparse spectrum {G}aussian process regression.
\newblock \emph{The Journal of Machine Learning Research}, 11:\penalty0
  1865--1881, 2010.

\bibitem[Le~Gall(2016)]{le2016brownian}
Jean-Fran{\c{c}}ois Le~Gall.
\newblock \emph{Brownian Motion, Martingales, and Stochastic Calculus}, volume
  274 of \emph{Graduate Texts in Mathematics}.
\newblock Springer, 2016.

\bibitem[Lee(2012)]{lee2012smooth}
John~M Lee.
\newblock \emph{Introduction to Smooth Manifolds}, volume 218 of \emph{Graduate
  Texts in Mathematics}.
\newblock Springer, 2012.

\bibitem[Lundberg and Lee(2017)]{lundberg2017unified}
Scott~M Lundberg and Su-In Lee.
\newblock A unified approach to interpreting model predictions.
\newblock In \emph{Advances in Neural Information Processing Systems},
  volume~30, 2017.
\newblock URL
  \url{https://proceedings.neurips.cc/paper_files/paper/2017/file/8a20a8621978632d76c43dfd28b67767-Paper.pdf}.

\bibitem[Lundstrom et~al.(2022)Lundstrom, Huang, and
  Razaviyayn]{lundstrom2022rigorous}
Daniel~D Lundstrom, Tianjian Huang, and Meisam Razaviyayn.
\newblock A rigorous study of integrated gradients method and extensions to
  internal neuron attributions.
\newblock In \emph{International Conference on Machine Learning}, pages
  14485--14508. PMLR, 2022.

\bibitem[Mitchell et~al.(2022)Mitchell, Cooper, Frank, and
  Holmes]{mitchell2022sampling}
Rory Mitchell, Joshua Cooper, Eibe Frank, and Geoffrey Holmes.
\newblock Sampling permutations for {S}hapley value estimation.
\newblock \emph{The Journal of Machine Learning Research}, 23\penalty0
  (1):\penalty0 2082--2127, 2022.

\bibitem[Murphy(2022)]{murphy2022probabilistic}
Kevin~P. Murphy.
\newblock \emph{Probabilistic Machine Learning: An Introduction}.
\newblock MIT Press, 2022.
\newblock URL \url{probml.ai}.

\bibitem[Neal(1996)]{neal1996bayesian}
Radford~M Neal.
\newblock \emph{Bayesian {L}earning for {N}eural {N}etworks}, volume 118 of
  \emph{Lecture Notes in Statistics}.
\newblock Springer, 1996.

\bibitem[Peng et~al.(2022)Peng, Li, Tsang, Zhu, Lv, and Zhou]{peng2022xai}
Xi~Peng, Yunfan Li, Ivor~W Tsang, Hongyuan Zhu, Jiancheng Lv, and Joey~Tianyi
  Zhou.
\newblock {XAI} beyond classification: Interpretable neural clustering.
\newblock \emph{The Journal of Machine Learning Research}, 23\penalty0
  (1):\penalty0 227--254, 2022.

\bibitem[Rasmussen and Williams(2006)]{rasmussen2006gaussian}
Carl~E. Rasmussen and Christopher~K. Williams.
\newblock \emph{Gaussian {P}rocesses for {M}achine {L}earning}.
\newblock MIT Press, 2006.
\newblock URL \url{https://gaussianprocess.org/gpml/}.

\bibitem[Ribeiro et~al.(2016)Ribeiro, Singh, and Guestrin]{ribeiro2016model}
Marco~Tulio Ribeiro, Sameer Singh, and Carlos Guestrin.
\newblock Model-agnostic interpretability of machine learning.
\newblock \emph{arXiv preprint arXiv:1606.05386}, 2016.

\bibitem[Riutort-Mayol et~al.(2023)Riutort-Mayol, B{\"u}rkner, Andersen, Solin,
  and Vehtari]{riutort2023practical}
Gabriel Riutort-Mayol, Paul-Christian B{\"u}rkner, Michael~R Andersen, Arno
  Solin, and Aki Vehtari.
\newblock Practical {H}ilbert space approximate {B}ayesian {G}aussian processes
  for probabilistic programming.
\newblock \emph{Statistics and Computing}, 33\penalty0 (1):\penalty0 17, 2023.

\bibitem[Robnik-{\v{S}}ikonja and Kononenko(2008)]{robnik2008explaining}
Marko Robnik-{\v{S}}ikonja and Igor Kononenko.
\newblock Explaining classifications for individual instances.
\newblock \emph{IEEE Transactions on Knowledge and Data Engineering},
  20\penalty0 (5):\penalty0 589--600, 2008.

\bibitem[Rudin(2019)]{rudin2019stop}
Cynthia Rudin.
\newblock Stop explaining black box machine learning models for high stakes
  decisions and use interpretable models instead.
\newblock \emph{Nature Machine Intelligence}, 1\penalty0 (5):\penalty0
  206--215, 2019.

\bibitem[Samek et~al.(2021)Samek, Montavon, Lapuschkin, Anders, and
  M{\"u}ller]{samek2021explaining}
Wojciech Samek, Gr{\'e}goire Montavon, Sebastian Lapuschkin, Christopher~J
  Anders, and Klaus-Robert M{\"u}ller.
\newblock Explaining deep neural networks and beyond: A review of methods and
  applications.
\newblock \emph{Proceedings of the IEEE}, 109\penalty0 (3):\penalty0 247--278,
  2021.

\bibitem[Seitz(2022)]{seitz2022gradient}
Sarem Seitz.
\newblock Gradient-based explanations for {G}aussian process regression and
  classification models.
\newblock \emph{arXiv preprint arXiv:2205.12797}, 2022.

\bibitem[Shapley(1953)]{shapley1953value}
Lloyd~S Shapley.
\newblock A value for n-person games.
\newblock \emph{Contributions to the Theory of Games}, pages 307--317, 1953.

\bibitem[Silva et~al.(2020)Silva, Gombolay, Killian, Jimenez, and
  Son]{silva2020optimization}
Andrew Silva, Matthew Gombolay, Taylor Killian, Ivan Jimenez, and Sung-Hyun
  Son.
\newblock Optimization methods for interpretable differentiable decision trees
  applied to reinforcement learning.
\newblock In \emph{International conference on artificial intelligence and
  statistics}, pages 1855--1865. PMLR, 2020.

\bibitem[Simonyan et~al.(2014)Simonyan, Vedaldi, and
  Zisserman]{simonyan2014deep}
Karen Simonyan, Andrea Vedaldi, and Andrew Zisserman.
\newblock Deep inside convolutional networks: visualising image classification
  models and saliency maps.
\newblock In \emph{Proceedings of the International Conference on Learning
  Representations (ICLR)}. ICLR, 2014.

\bibitem[Solak et~al.(2002)Solak, Murray-Smith, Leithead, Leith, and
  Rasmussen]{solak2002derivative}
Ercan Solak, Roderick Murray-Smith, WE~Leithead, D~Leith, and Carl Rasmussen.
\newblock Derivative observations in {G}aussian process models of dynamic
  systems.
\newblock \emph{Advances in Neural Information Processing Systems}, 15, 2002.

\bibitem[Springenberg et~al.(2014)Springenberg, Dosovitskiy, Brox, and
  Riedmiller]{springenberg2014striving}
Jost~Tobias Springenberg, Alexey Dosovitskiy, Thomas Brox, and Martin
  Riedmiller.
\newblock Striving for simplicity: The all convolutional net.
\newblock \emph{arXiv preprint arXiv:1412.6806}, 2014.

\bibitem[Stein(1999)]{stein1999interpolation}
Michael~L Stein.
\newblock \emph{Interpolation of Spatial Data: Some Theory for Kriging}.
\newblock Springer Science \& Business Media, 1999.

\bibitem[{\v{S}}trumbelj and Kononenko(2014)]{strumbelj2014explaining}
Erik {\v{S}}trumbelj and Igor Kononenko.
\newblock Explaining prediction models and individual predictions with feature
  contributions.
\newblock \emph{Knowledge and Information Systems}, 41\penalty0 (3):\penalty0
  647--665, 2014.

\bibitem[{\v{S}}trumbelj et~al.(2009){\v{S}}trumbelj, Kononenko, and
  {\v{S}}ikonja]{strumbelj2009explaining}
Erik {\v{S}}trumbelj, Igor Kononenko, and M~Robnik {\v{S}}ikonja.
\newblock Explaining instance classifications with interactions of subsets of
  feature values.
\newblock \emph{Data \& Knowledge Engineering}, 68\penalty0 (10):\penalty0
  886--904, 2009.

\bibitem[Sturmfels et~al.(2020)Sturmfels, Lundberg, and
  Lee]{sturmfels2020visualizing}
Pascal Sturmfels, Scott Lundberg, and Su-In Lee.
\newblock Visualizing the impact of feature attribution baselines.
\newblock \emph{Distill}, 5\penalty0 (1):\penalty0 e22, 2020.
\newblock URL \url{https://distill.pub/2020/attribution-baselines/}.

\bibitem[Sundararajan and Najmi(2020)]{sundararajan2020many}
Mukund Sundararajan and Amir Najmi.
\newblock The many {S}hapley values for model explanation.
\newblock In \emph{International Conference on Machine Learning}, pages
  9269--9278. PMLR, 2020.

\bibitem[Sundararajan et~al.(2017)Sundararajan, Taly, and
  Yan]{sundararajan2017axiomatic}
Mukund Sundararajan, Ankur Taly, and Qiqi Yan.
\newblock Axiomatic attribution for deep networks.
\newblock In \emph{International Conference on Machine Learning}, pages
  3319--3328. PMLR, 2017.

\bibitem[van~der Waa et~al.(2021)van~der Waa, Nieuwburg, Cremers, and
  Neerincx]{van2021evaluating}
Jasper van~der Waa, Elisabeth Nieuwburg, Anita Cremers, and Mark Neerincx.
\newblock Evaluating {XAI}: A comparison of rule-based and example-based
  explanations.
\newblock \emph{Artificial Intelligence}, 291:\penalty0 103404, 2021.

\bibitem[Voita et~al.(2019)Voita, Talbot, Moiseev, Sennrich, and
  Titov]{voita2019analyzing}
Elena Voita, David Talbot, Fedor Moiseev, Rico Sennrich, and Ivan Titov.
\newblock Analyzing multi-head self-attention: Specialized heads do the heavy
  lifting, the rest can be pruned.
\newblock In \emph{Proceedings of the 57th Annual Meeting of the Association
  for Computational Linguistics}. Association for Computational Linguistics,
  2019.

\bibitem[Wahba(1990)]{wahba1990spline}
Grace Wahba.
\newblock \emph{Spline Models for Observational Data}.
\newblock SIAM, 1990.

\bibitem[Wilson et~al.(2016)Wilson, Hu, Salakhutdinov, and
  Xing]{wilson2016deep}
Andrew~Gordon Wilson, Zhiting Hu, Ruslan Salakhutdinov, and Eric~P Xing.
\newblock Deep kernel learning.
\newblock In \emph{Artificial Intelligence and Statistics}, pages 370--378.
  PMLR, 2016.

\bibitem[Wolberg et~al.(1995)Wolberg, Street, and
  Mangasarian]{misc_breast_cancer_wisconsin_(prognostic)_16}
William Wolberg, W.~Street, and Olvi Mangasarian.
\newblock {Breast Cancer Wisconsin (Prognostic)}.
\newblock UCI Machine Learning Repository, 1995.
\newblock {DOI}: https://doi.org/10.24432/C5GK50.

\bibitem[Ying et~al.(2019)Ying, Bourgeois, You, Zitnik, and
  Leskovec]{ying2019gnnexplainer}
Zhitao Ying, Dylan Bourgeois, Jiaxuan You, Marinka Zitnik, and Jure Leskovec.
\newblock {GNNExplainer}: Generating explanations for graph neural networks.
\newblock \emph{Advances in Neural Information Processing Systems}, 32, 2019.

\bibitem[Yoshikawa and Iwata(2021)]{yoshikawa2021gaussian}
Yuya Yoshikawa and Tomoharu Iwata.
\newblock Gaussian process regression with interpretable sample-wise feature
  weights.
\newblock \emph{IEEE Transactions on Neural Networks and Learning Systems},
  2021.

\bibitem[Zhang et~al.(2018)Zhang, Nakadai, and Fukumizu]{zhang2018black}
Hao Zhang, Shinji Nakadai, and Kenji Fukumizu.
\newblock From black-box to white-box: Interpretable learning with kernel
  machines.
\newblock In \emph{Machine Learning and Data Mining in Pattern Recognition:
  14th International Conference, MLDM 2018, New York, NY, USA, July 15-19,
  2018, Proceedings, Part I 14}, pages 213--227. Springer, 2018.

\end{thebibliography}

\end{document}